\documentclass{article}

\usepackage{iclr2027_conference,times}
\iclrfinalcopy
\usepackage[utf8]{inputenc}
\usepackage[T1]{fontenc}
\usepackage{hyperref}
\usepackage{url}
\usepackage{booktabs}
\usepackage{amsfonts}
\usepackage{nicefrac}
\usepackage{microtype}
\usepackage{xcolor}
\usepackage{amsmath}
\usepackage{amsthm}
\usepackage{amssymb}
\usepackage{graphicx}
\usepackage{subcaption}
\usepackage{svg}
\usepackage{wrapfig}
\usepackage{enumitem}
\usepackage{adjustbox}
\newtheorem{proposition}{Proposition}
\newtheorem{theorem}{Theorem}

\title{A Generative Model of Complex Networks Using Graphons and Neural Inverse Operators}

\author{%
  Wooseong Choi\thanks{Equal contribution.} \\
  Department of Biomedical Engineering \\
  University of Southern California \\
  Los Angeles, CA 90089 \\
  \texttt{wooseong@usc.edu}
  \And
  Italo'Ivo Lima Dias Pinto\footnotemark[1] \\
  Department of Electrical Engineering \\
  University of Southern California \\
  Los Angeles, CA 90089 \\
  \texttt{italoivo@gmail.com}
  \And
  Chen Sun \\
  Department of Biomedical Engineering \\
  University of Southern California \\
  Los Angeles, CA 90089 \\
  \texttt{csun8109@usc.edu}
  \And
  Gaurav Gupta\thanks{Work done outside of the author's role at Amazon.} \\
   AWS AI\\
   Santa Clara, CA 95050 \\
  \texttt{gaurav71531@gmail.com }
  \And
  Dong Song \thanks{Co-corresponding authors.} \\
  Department of Biomedical Engineering \\
  University of Southern California \\
  Los Angeles, CA 90089 \\
  \texttt{dsong@usc.edu}
  \And
  Paul Bogdan\footnotemark[3] \\
  Department of Electrical Engineering \\
  University of Southern California \\
  Los Angeles, CA 90089 \\
  \texttt{pbogdan@usc.edu}
}
\usepackage{placeins}

\begin{document}

\maketitle
\lhead{Preprint. Under review.}

\begin{abstract}
Generative graph models are central to understanding and simulating complex networks. However, existing approaches have complementary strengths and limitations. Mechanistic models offer interpretability but rely on instance-specific estimation methods. Deep generative models, on the other hand, offer amortized inference at the cost of interpretability and are largely limited to graph sizes seen during training. Scientific applications motivate a framework that retains the strengths of both paradigms. We bridge them by formulating both the generative model and parameter recovery in function space. A multifractal step graphon extends standard step graphons with a recursive construction that compactly parameterizes complex networks. This formulation admits a neural inverse operator to recover its parameters, enabling inference on unseen graph sizes. We evaluate our model, trained only on synthetic multifractal step graphon realizations, against both paradigms. Against a graph foundation model pretrained on empirical networks, our method achieves the best average performance on three of four metrics in a zero-shot graph-generation benchmark, indicating that the model transfers to real-world graphs. We also apply our method to single-observation networks, a regime largely inaccessible to deep models that require training corpora, where it performs comparably to an instance-specific method that optimizes on each graph. In a multi-subject EEG case study, the inferred parameters track a reversible change in brain state more sensitively than traditional network statistics. Together, these results indicate that mechanistic interpretability and amortized inference can be effectively unified in a generative graph model to enhance our understanding of complex networks.

\end{abstract}

\section{Introduction}

Complex networks model biological \citep{la2018self, yin2020network, yang2021hidden}, social \citep{lee2006statistical, liu2017fractal}, and infrastructure systems \citep{long2021multifractal, batty1994fractal}. Summary statistics of such networks aggregate features across scales and can obscure differences in the underlying structure \citep{xue2017reliable}. In addition, these metrics can vary spuriously across different partial observations of the same network \citep{lee2006sampled,stumpf2005subnets, achlioptas2005bias}. This motivates going beyond descriptive statistics and recovering the compact generative mechanisms that reproduce the observed network. Such mechanisms are particularly valuable for revealing hidden differences underlying complex networks \citep{ yang2020controlling, akarca2021generative}. As such, recovering these generative processes is essential for enhancing our understanding of complex systems.

Generative graph models synthesize realistic networks by capturing meaningful representations of their underlying structure. Two paradigms have emerged, each with its own approach. Mechanistic models define explicit wiring rules that are algorithmically implemented to produce graphs with prescribed properties \citep{erdos1960evolution, barabasi1999emergence}. In contrast, deep generative models learn graph distributions directly from data \citep{li2018learning, kipf2016variational}.

Classical mechanistic models capture specific regimes, from homogeneous random graphs \citep{erdos1960evolution} to small-world \citep{watts1998collective} and preferential-attachment networks \citep{barabasi1999emergence}. However, these frameworks are restricted to modeling a prescribed class of networks. To overcome these constraints, recursive models such as SKG  \citep{leskovec2010kronecker}, multifractal network generators \citep{palla2010multifractal}, and WMGM \citep{yang2021hidden} established the foundation for using recursive rules to compactly describe networks observed across many domains.

However, recovering these models from data remains challenging. SKGs \citep{leskovec2010kronecker} rely on maximum likelihood estimation (MLE) with permutation sampling. Multifractal network generators \citep{palla2010multifractal} use simulated annealing, while the weighted multifractal generative model (WMGM) \citep{yang2021hidden} relies on expectation-maximization (EM). MLE and EM are susceptible to local optima \citep{wu1983convergence, blei2017variational}, whereas simulated annealing suffers from prohibitively slow convergence \citep{hajek1988cooling}. These methods also search over latent node-to-community assignments, a combinatorial step whose cost grows rapidly with graph size \citep{leskovec2010kronecker, daudin2008mixture, brandes2008modularity}. Moreover, these methods are instance-specific, so the aforementioned costs recur with every observation.

Deep generative graph models \citep{simonovsky2018graphvae, you2018graphrnn, liao2019efficient, vignac2022digress} amortize generation from data, but their learned representations are often less interpretable and typically do not extrapolate to graph sizes outside the training distribution. Autoregressive models \citep{li2018learning, you2018graphrnn, liao2019efficient} cannot extrapolate beyond the training distribution due to architectural constraints (e.g., GraphRNN's BFS window and GRAN's maximum node capacity). Variational autoencoders and diffusion-based models similarly generate graphs whose sizes are effectively limited to those observed during training \citep{simonovsky2018graphvae, niu2020permutation, vignac2022digress}.

\textbf{This work.} Taken together, there is a strong need for a framework that retains the strengths of both paradigms, such as interpretability and amortized inference, which are jointly necessary for generative models to serve in scientific applications. Generalization to unseen graph sizes further broadens its applicability. We address this need by formulating both the generative model and its inference in function space:

\begin{itemize}[leftmargin=*]
\item\textbf{Compact Function-Space Representation of Complex Networks:} We introduce the multifractal step graphon: a recursively-refined, piecewise-constant kernel parameterized by a small set of base coefficients $\theta^{(1)} = (l^{(1)}, p^{(1)})$ that directly encode hierarchical community structure.
\item\textbf{Neural Inverse Operator Formulation:} We cast the parametric graphon recovery as an inverse operator-learning problem. This formulation enables amortized inference that generalizes to unseen graph sizes.
\end{itemize}

\begin{figure}[t]
    \centering
    \includegraphics[width=\textwidth]{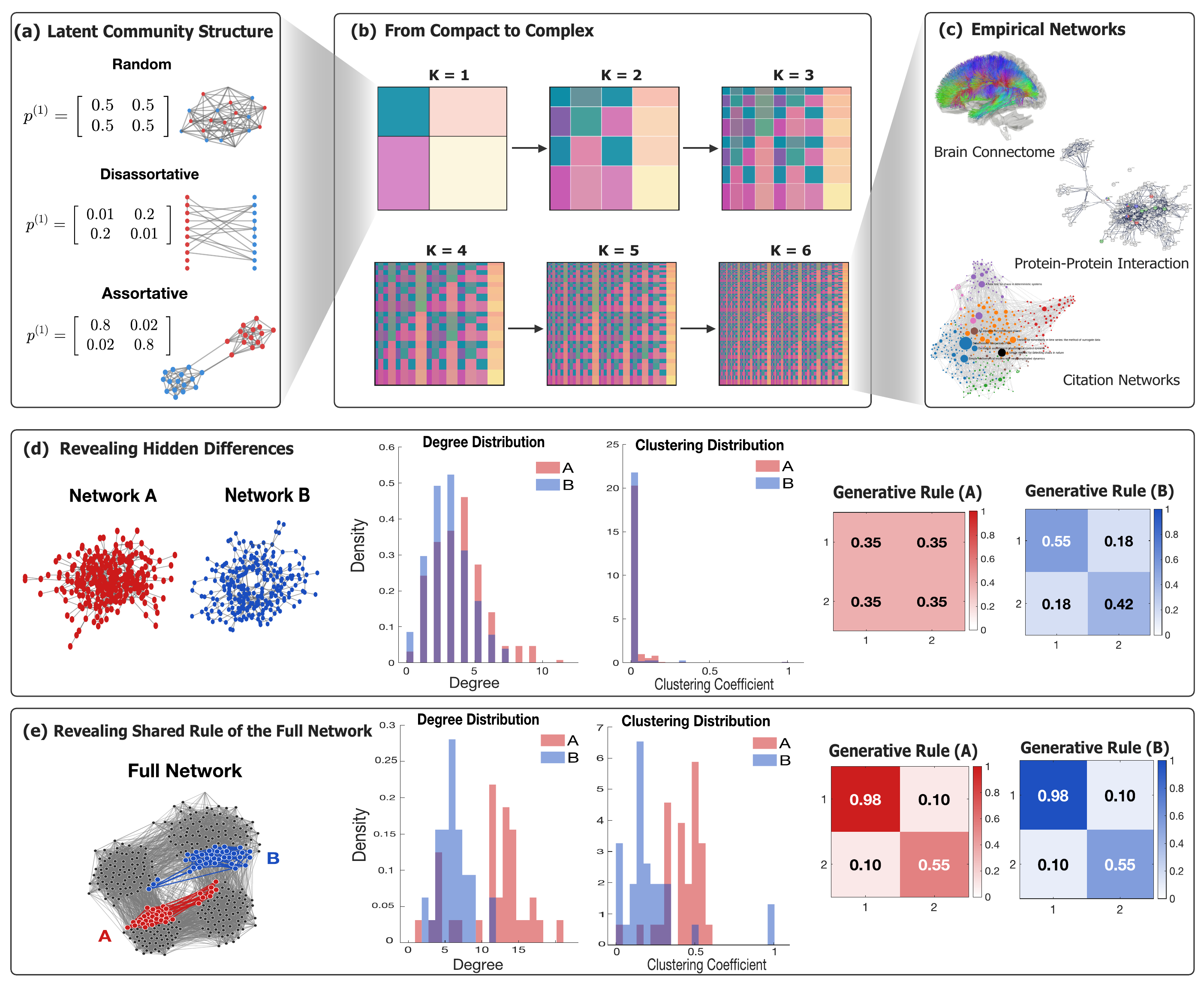}
    \caption{\textbf{Discovery of Hidden Structures in Real-World Networks} \textbf{(a)} The base graphon encodes hidden community structures underlying complex networks. \textbf{(b)} Kronecker expansion of base parameters produces the multifractal step graphon. \textbf{(c)} Finite realization of the graphon produces networks that share similar properties with those seen in various empirical domains \protect\footnotemark. Recovering the generative rule reveals \textbf{(d)} hidden differences between networks with the same statistics and \textbf{(e)} the shared latent structure of a network from different partial observations, where subsampling can produce different statistics.}
    \label{fig:schematic}
    \vspace{-5mm}
\end{figure}

\footnotetext{Figure \ref{fig:schematic} incorporates a brain tractography image by \citet{yeh2018population} [\href{https://creativecommons.org/licenses/by-sa/4.0/}{CC BY-SA 4.0}] and a protein interaction network by \citet{chen2019construction} [\href{https://creativecommons.org/licenses/by/4.0/}{CC BY 4.0}] via Wikimedia Commons. The citation network by citationgraph.org is Public Domain (CC0). This derived schematic is distributed under CC BY-SA 4.0.}
\vspace{-1mm}

\section{Related Work}
\paragraph{Graphons.} A graphon is a bounded, measurable function $W: [0,1]^2 \to [0,1]$ encoding connection probabilities between latent node positions \citep{lovasz2006limits, borgs2008convergent}. Step graphons approximate this continuous space as piecewise-constant surfaces, admitting an interpretable parametrization in which latent communities and connection rules correspond to domain partitions and uniform probabilities within each block, respectively \citep{airoldi2013stochastic}.  Prior work on graphon estimation has been limited to instance-specific statistical methods \citep{airoldi2013stochastic, chan2014consistent}. Moreover, standard step graphons implicitly assume non-hierarchical topologies and cannot represent the recursive, self-similar structure present in many real-life networks \citep{xue2017reliable}.

\paragraph{Neural Operators.} Recently, neural operators have emerged as powerful frameworks for learning the mapping between function spaces, with applications primarily in solving forward partial differential equations \citep{lu2021learning, kovachki2023neural, li2021fourier, gupta2021multiwavelet, xiao2023coupled}. \emph{Inverse} operator learning, in which the network recovers underlying parameters from functions, has been applied to PDE inverse problems \citep{molinaro2023neural}. However, no other work has applied it directly to graphon estimation to the best of our knowledge.

\paragraph{Graph Generation.} The complementary limitations of mechanistic and deep generative graph models are discussed in Section~1. Existing graphon-based generative models offer either interpretability, through a linear factorization over learned graphon factors in a graphon autoencoder (GNAE) \citep{xu2021learning} or flexibility across graph sizes by representing graphons as implicit neural representations at arbitrary resolution \citep{xia2023implicit, azizpour2024scalable}. Our approach provides both. The method's parameters define an explicit wiring rule with a closed-form link to network properties (Appendix~\ref{app:closed-form}), whereas GNAE's interpretability lies in its learned latent factors. Its function-space formulation further admits a neural inverse operator that generalizes across graph sizes.

\section{Multifractal Step Graphon Formulation}

\subsection{Overview}
Here, we extend step graphons \citep{airoldi2013stochastic} to better represent hierarchical structure observed in real-life networks \citep{xue2017reliable}. We formalize this through a recursively refined piecewise-constant kernel on the unit square, which gives a continuum model parameterized by $\theta^{(1)} = (l^{(1)}, p^{(1)})$. The parameters $l^{(1)}$ and $p^{(1)}$ yield heterogeneous node attribution and weighted connectivity, respectively. $M$ defines the number of base communities. A Kronecker recursion expands these base parameters to arbitrary depth $K$, partitioning the latent space into $M^K$ communities whose sizes and interaction strengths factorize multiplicatively across scales. We use $M=2$ throughout, which keeps the $M^K$ communities populated in tractable training graphs. Because two recursion steps together form a valid $M=4$ base rule, this choice still covers the subset of larger-$M$ models whose base rule is Kronecker-factorizable (Appendix~\ref{app:choice_M}).

We model generation of graphs from the multifractal step graphon as realizations of a jointly exchangeable array. By the Aldous representation theorem \citep{aldous1981representations}, any such array admits a generative representation as a measurable function $W: [0,1]^2 \to [0,1]$ from which finite graphs are sampled. This provides the foundation for graphon-based generative modeling. Thus, our use of the graphon follows the Aldous tradition---as a generative kernel from which finite networks are sampled---rather than the Lovász–Szegedy tradition of graph limits.

\subsubsection{Continuous Latent Space}
A graphon is a bounded, Lebesgue-measurable function $W: [0,1]^2 \to [0,1]$ belonging to the Hilbert space $L^2([0,1]^2)$ \citep{lovasz2006limits, borgs2008convergent}. While traditionally defined as symmetric ($W(x,y) = W(y,x)$) for undirected graphs, we follow theoretical \citep{aldous1981representations} and empirical \citep{airoldi2013stochastic} precedents and generalize to asymmetric kernels ($W(x,y) \neq W(y,x)$) to capture directed topologies. In this formulation, $W(x,y)$ represents the probability of a directed edge from a node at latent coordinate $x$ to a node at $y$.

\subsubsection{Step Graphon Parameterization}
A step graphon approximates this continuous space as a piecewise-constant surface. We partition the unit interval $[0,1]$ into $M$ disjoint subsets $\{I_1, \dots, I_M\}$. A probability vector $l \in \Delta^{M-1}$ defines the length of each subset such that $\lambda(I_i) = l_i$. The connection probability between subsets is defined by a blockwise affinity matrix $p \in [0,1]^{M \times M}$. The step graphon is then:
\begin{equation} \label{eq:step_graphon}
    W(x,y) = \sum_{i=1}^M \sum_{j=1}^M p_{ij}\, \mathbb{I}_{I_i}(x)\, \mathbb{I}_{I_j}(y)
\end{equation}
where $\mathbb{I}$ is the indicator function.

\subsubsection{Multifractal Connectivity via Kronecker Recursion}
\label{sec:multifractal-extension}

Standard step graphons implicitly assume non-hierarchical topologies. However, real-world networks exhibit varying levels of heterogeneity across different scales \citep{xue2017reliable}. This heterogeneity is formalized by the Lipschitz--H\"older exponent $\alpha$, which measures how the local connectivity scales with resolution. A \emph{monofractal} object has the same scaling behavior everywhere (a single $\alpha$), while a \emph{multifractal} object exhibits a distribution of scaling behaviors, characterized by the spectrum $f(\alpha)$ as its width ($\Delta\alpha = \alpha_{\max} - \alpha_{\min}$) and shape.

To capture this self-similar structure, we recursively expand the base parameters $\theta^{(1)} = (l^{(1)}, p^{(1)})$, where $l^{(1)} \in \text{int}(\Delta^{M-1})$ and $p^{(1)} \in (0,1)^{M \times M}$, to a multiscale depth $K$ via the Kronecker product:
\begin{equation}
l^{(K)} = \underbrace{l^{(1)} \otimes \cdots \otimes l^{(1)}}_{K \text{ times}}, \quad p^{(K)} = \underbrace{p^{(1)} \otimes \cdots \otimes p^{(1)}}_{K \text{ times}}
\end{equation}
This recursive expansion partitions the space into $M^K$ contiguous sub-intervals indexed by $q = (q_1, \ldots, q_K)$. The multifractal step graphon is defined as:
\begin{equation}
W_K(x, y) = \sum_q \sum_h p^{(K)}_{qh} \, \mathbb{I}_{I_q}(x) \, \mathbb{I}_{I_h}(y)
\end{equation}
where the interaction amplitude factorizes across scales as $p^{(K)}_{qh} = \prod_{k=1}^{K} p^{(1)}_{q_k h_k}$. This multiplicative factorization enables varying levels of heterogeneity across scales: $\log p^{(K)}_{qh}$ is a sum of $K$ terms from $\log p^{(1)}$, so different index paths yield different local scaling exponents. Uniform entries in $p^{(1)}$ collapse to a single exponent ($\Delta\alpha = 0$, monofractal), while heterogeneous entries widen the spectrum. Thus, $\theta^{(1)}$ compactly controls where $W_K$ lies on the monofractal--multifractal continuum. The Kronecker structure also yields closed-form expressions for several network statistics, such as the multifractal spectrum, derived in Appendix~\ref{app:closed-form}.

\begin{figure}[t]
    \centering
    \includegraphics[width=\linewidth]{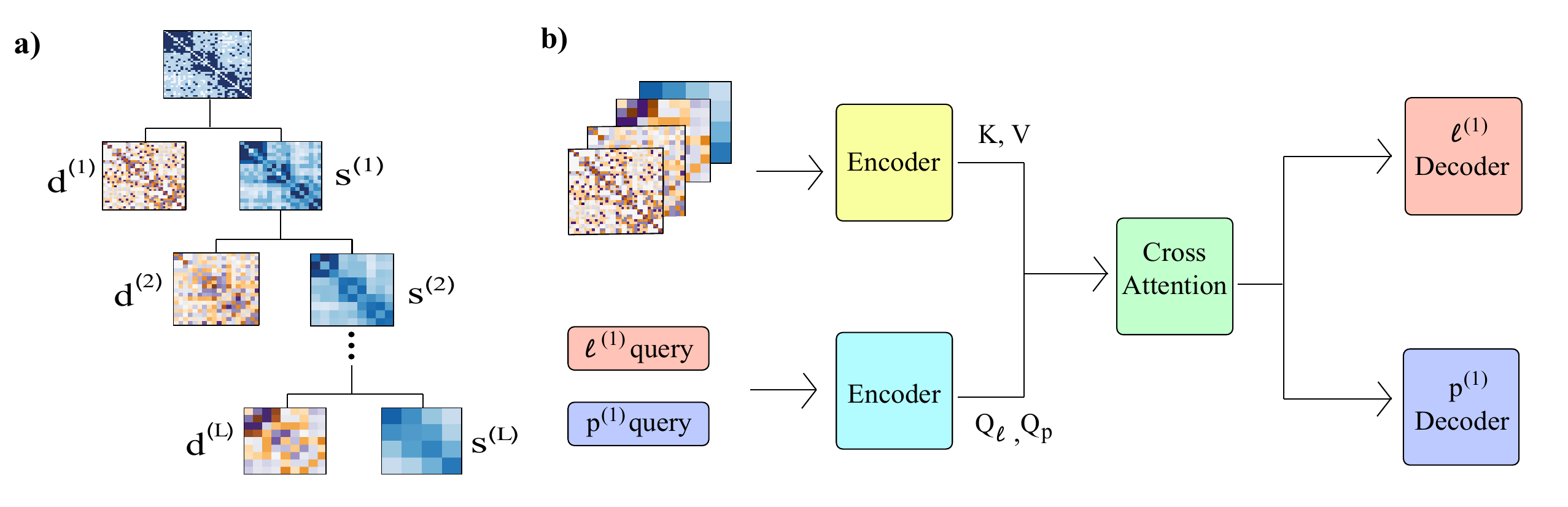}
    \caption{
        \textbf{Neural inverse operator architecture.} \textbf{(a)} The input adjacency matrix is decomposed across $L$ scales by a 2D multiwavelet transform, yielding hierarchical coefficients that are tokenized and processed by a branch network (Keys and Values). \textbf{(b)} Learnable queries for $\ell^{(1)}$ and $p^{(1)}$ pass through a trunk network and are combined with the branch via cross-attention. Scale-specific outputs are decoded by MLPs to predict the base parameters.
    }
    \label{fig:mwnot_architecture}
    \vspace{-4mm}
\end{figure}

\subsubsection{Sampling Networks from the Multifractal Step Graphon}

A finite graph $G = (V, E)$ is constructed by sampling from the multifractal
step graphon $W_K$. By the Aldous representation, each node $u$ is associated
with a latent coordinate $x_u \sim \text{Uniform}(0, 1)$, which falls in
the interval $I_{q_u}$ with probability $l^{(K)}_{q_u}$. Because $W_K$ is
piecewise-constant on these intervals, this is equivalent to drawing
community assignments directly:
\begin{equation}
    q_u \sim \text{Multinomial}\bigl(l^{(K)}\bigr),
\end{equation}
which is the form used in implementation. The continuous interaction
intensity for pair $(u, v)$ is then the constant value of $W_K$ on the
block $I_{q_u} \times I_{q_v}$:
\begin{equation}
    W_{uv} = W_K(x_u, x_v) = p^{(K)}_{q_u q_v}.
\end{equation}
To map this continuous amplitude to a discrete edge weight $r_{uv} \in \mathbb{N}_0$,
we sample from a conditional geometric distribution:
\begin{equation}
    P(r_{uv} = r \mid W_{uv}) = (1 - W_{uv})(W_{uv})^r.
\end{equation}
Higher connection probabilities naturally produce heavier-tailed edge weights
while maintaining support on all non-negative integers.

\subsubsection{Identifiability}
Inverse problems are typically ill-posed \citep{molinaro2023neural}. We show identifiability (Proposition 1) of our formulation, which ensures that the inverse map is well-defined and distinct parameters produce distinct graphons.

\begin{proposition}[Identifiability of the multifractal step graphon]
\label{prop:identifiability}
Assume additionally that the expanded block profiles of $p^{(K)}$ are pairwise distinct, so that no two blocks have identical incoming and outgoing interaction profiles. Under this generic non-degeneracy condition, let $M \geq 2$ and $K \geq 2$. Consider two base parameterizations $\theta_1 = (l_1^{(1)}, p_1^{(1)})$ and $\theta_2 = (l_2^{(1)}, p_2^{(1)})$, where $l_i \in \mathrm{int}(\Delta^{M-1})$ and $p_i \in (0,1)^{M \times M}$. Let $W_K^{(1)}$ and $W_K^{(2)}$ denote the corresponding multifractal step graphons obtained via $K$-fold Kronecker expansion. If $W_K^{(1)} = W_K^{(2)}$ almost everywhere on $[0,1]^2$, then there exists a permutation $\sigma \in S_M$ such that $l_1^{(1)} = \sigma(l_2^{(1)})$ and $p_1^{(1)} = \sigma(p_2^{(1)})$, where $\sigma$ acts by reindexing: $\sigma(l)_i = l_{\sigma(i)}$ and $\sigma(p)_{ij} = p_{\sigma(i)\sigma(j)}$.
\end{proposition}

The proof, given in Appendix~\ref{app:identifiability}, proceeds in three steps: (i) the piecewise-constant graphon determines its block structure up to permutation via its discontinuities. (ii) the Kronecker factorization of the expanded parameters into base parameters is unique, by rank-1 symmetric tensor decomposition. (iii) the coupling between $l^{(K)}$ and $p^{(K)}$ through shared block indices forces the permutation to be consistent across both.

\section{Inference via Neural Operators}
\label{sec:inference}
Recovering $\theta^{(1)} = (l^{(1)}, p^{(1)})$ from an observed graph requires learning the inverse operator $\mathcal{F}: W_K \mapsto \theta^{(1)}$. The observed adjacency matrix is treated as a finite-sample realization of a graph generated from $W_K \in L^2([0,1]^2)$. Because $W_1$ is fully parameterized by $\theta^{(1)}$, recovering these parameters fully specifies the multiscale graphon $W_K$.

\begin{figure}[t]
    \centering
    \includegraphics[width=\linewidth]{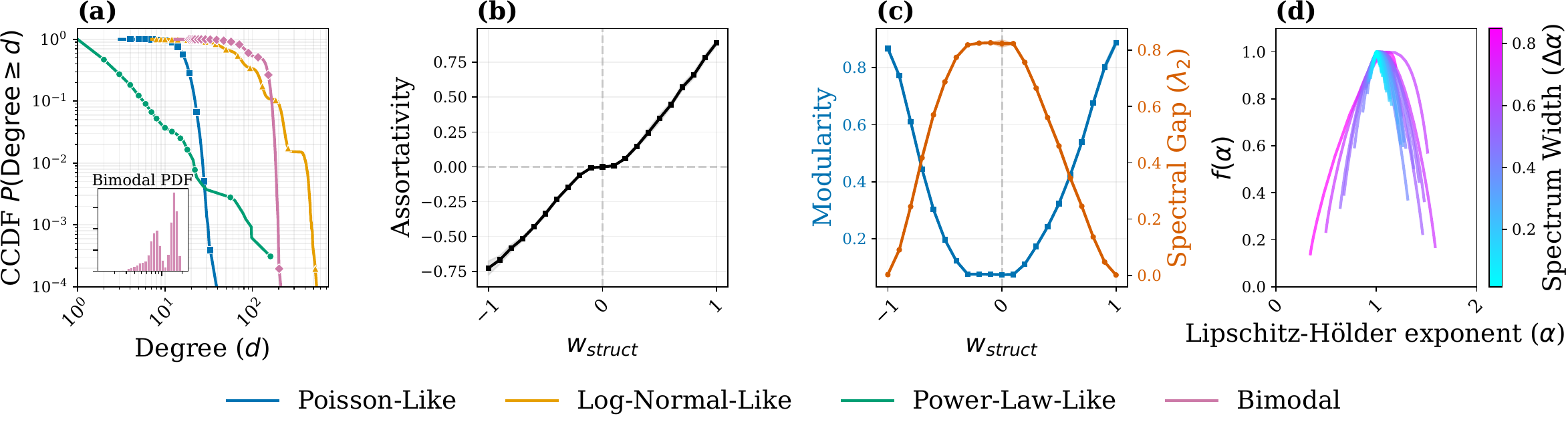}
    \caption{\textbf{Compact Parameterization of Diverse Topologies.} \textbf{(a)} Degree distributions span Poisson-like, log-normal-like, power-law-like (with exponential-tail), and bimodal regimes that are present in real-life networks. \textbf{(b)} Sweeping $w_{\mathrm{struct}}$ allows expression of varying levels of assortativity and disassortativity. \textbf{(c)} $w_{\mathrm{struct}}$ also controls the full range of modularity (left axis) and spectral gap $\lambda_2$ (right axis). \textbf{(d)} Multifractal spectra forms a continuous envelope from monofractal to multifractal.}
    \label{fig:expressivity}
    \vspace{-3mm}
\end{figure}

Multifractal step graphons are low-regularity, piecewise-constant objects, for which multiwavelet bases are known to provide efficient representations \citep{alpert1993class}. We verify empirically that the multiwavelet basis resolves the recursive structure of $W_K$: the proportion of energy in the detail-coefficients grows exponentially with recursion depth (Appendix~\ref{app:wavelet-validation}). We instantiate the inverse operator as a Multiwavelet Neural Operator Transformer (MWNOT) (Fig.\ref{fig:mwnot_architecture}), combining a 2D multiwavelet decomposition \citep{gupta2021multiwavelet, xiao2023coupled} with cross-attention for operator approximation \citep{shih2025transformers}. We use MWNOT as one architecture compatible with the proposed formulation, not as a claim of architectural optimality.

\subsection{Architecture}

First, a 2D discrete multiwavelet transform \citep{gupta2021multiwavelet} decomposes $A$ across $L$ spatial scales. The multiwavelet basis is constructed from shifted Legendre polynomials of order $k$ and orthogonalized via Gram--Schmidt so that detail coefficients at each scale are orthogonal to the coarse approximation. At each scale $\ell$, the transform yields coarse coefficients $s^{(\ell)}$ capturing macroscopic connectivity and detail coefficients $d^{(\ell)}$ capturing high-frequency variation.

Second, we adopt a branch--trunk cross-attention architecture \citep{lu2021learning, shih2025transformers} to map these multiscale coefficients to $\theta^{(1)}$. The branch network encodes the input function through its wavelet tokens, producing keys and values. The trunk network processes learnable query embeddings for $l^{(1)}$ and $p^{(1)}$. Cross-attention approximates the operator through a linear combination of input-derived features weighted by target-dependent queries \citep{shih2025transformers}.

Third, scale-specific attention outputs are concatenated across $\ell$ and decoded by MLPs to predict $\hat{l}^{(1)}$ and $\hat{p}^{(1)}$. Full architectural details are provided in Appendix~\ref{app:architecture}.

\subsection{Training}
Training details, including data generation, loss functions, optimization schedule, and hyperparameter settings, are provided in Appendix~\ref{app:training}.

\section{Experiments}
The central claim of this work is that the multifractal step graphon, paired with a neural inverse operator trained only on its synthetic realizations, recovers interpretable parameters of real-world networks through amortized inference that generalizes across graph sizes. We support this claim through a series of experiments evaluating expressivity (Sec.~\ref{sec:expressivity}), parameter recovery on unseen graph sizes (Sec.~\ref{sec:recovery}), and zero-shot transfer to empirical data of various domains and sizes (Sec.~\ref{sec:benchmarks}). Finally, we apply our framework to single-observation empirical networks (Sec.~\ref{sec:emp_data}), where the inferred parameters reconstruct each network comparably to per-instance fitting and track a reversible change in brain state in a multi-subject EEG case study.

\subsection{Compact Parameterization of Diverse Topologies}
\label{sec:expressivity}

We empirically validate that the multifractal step graphon family covers various topologies observed in real-world networks. We sweep two axes of $\mathbf{p}^{(1)}$: a heterogeneity parameter $w_{\mathrm{het}} \in [0,1]$ controlling block-density variation, and a structure parameter $w_{\mathrm{struct}} \in [-1,1]$ interpolating from disassortative to assortative community organization. Sweep details are in Appendix~\ref{app:expressivity_setup}.

Figure~\ref{fig:expressivity} confirms that varying $\boldsymbol{\theta}^{(1)}$ alone produces networks with several qualitatively distinct regimes of (a) degree distributions \citep{broido2019scale}, (b)  assortativity, (c) modularity with inverse relationship to spectral gap, and (d) multifractal spectrum width, spanning from monofractal to multifractal. Together, these results suggest that the model provides a broad synthetic pretraining distribution that may enable transfer to real-world networks in various domains.

\begin{wrapfigure}{r}{0.5\textwidth}
    \vspace{-10mm}
    \centering
    \includegraphics[width=0.5\textwidth]{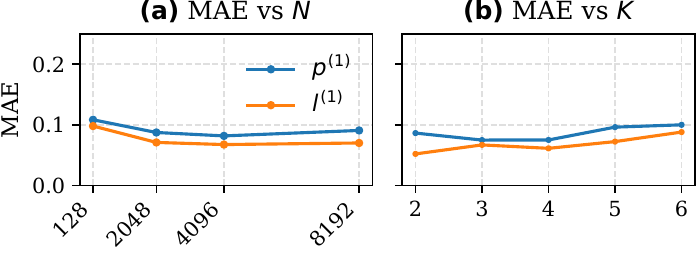}
    \caption{\textbf{Model Accuracy on Synthetic Graphs.} \textbf{(a)} The model generalizes to unseen graph sizes and \textbf{(b)} shows stable performance across varying levels of hierarchical structure.}
    \label{fig:recovery}
    \vspace{-3mm}
\end{wrapfigure}

\subsection{Zero-Shot Transfer to Unseen Size}
\label{sec:recovery}
We next verify that the neural inverse operator recovers $\boldsymbol{\theta}^{(1)}$ accurately across the conditions encountered at deployment: varying graph sizes $N$ and recursive depths $K$. A single MWNOT is trained on $N \in \{128, 512, 1024\}$ and $K \in \{2, 4, 6\}$, then evaluated on 100 held-out test samples per setting where performance is averaged across $K \in \{2, \ldots, 6\}$ for each $N$. We test our model's ability to generalize to unseen graph sizes. MAE for $\mathbf{p}^{(1)}$ and $\mathbf{l}^{(1)}$ remains stable as size is changed from $N=128$ to $N=8192$ (Fig.~\ref{fig:recovery}a). Consequently, our model allows the user to train on small graphs for inference on larger ones, which may enhance the broad applicability of the model. Recovery remains stable across hierarchical depths as well.

\begin{table}[t]
\caption{Zero-shot Performance Comparison}
\label{tab:twocolumn_results}
\centering
\resizebox{\textwidth}{!}{%
\begin{tabular}{llcccc | llcccc}
\toprule
\textbf{Domain} & \textbf{Method} & \textbf{DEG} & \textbf{CC} & \textbf{Spec} & \textbf{Orb} & \textbf{Domain} & \textbf{Method} & \textbf{DEG} & \textbf{CC} & \textbf{Spec} & \textbf{Orb} \\
\midrule
FB & DiGress & 0.3376 & 0.6298 & 0.0797 & 0.3593 & BIO & DiGress & 0.2712 & 0.5202 & 0.1127 & 0.3188 \\
& LGGM-X & 0.4723 & 0.6843 & 0.2924 & 0.7555 & & LGGM-X & 0.1081 & 0.2696 & 0.0900 & 0.2053 \\
& Ours-Mono & 0.2782 & 0.3045 & 0.0308 & 0.0412 & & Ours-Mono & 0.0879 & 0.1467 & 0.0827 & 0.1584 \\
& \textbf{Ours} & \textbf{0.0589} & \textbf{0.0500} & \textbf{0.0284} & 0.0403 & & \textbf{Ours} & \textbf{0.0442} & \textbf{0.0605} & \textbf{0.0629} & \textbf{0.1090} \\
\midrule
ASN & DiGress & 0.1496 & 0.3258 & 0.1506 & 0.4420 & ECON & DiGress & 0.2987 & 0.4841 & 0.2162 & 0.3834 \\
& LGGM-X & \textbf{0.0281} & 0.2440 & \textbf{0.0830} & \textbf{0.0618} & & LGGM-X & 0.1213 & 0.0920 & 0.1120 & 0.1086 \\
& Ours-Mono & 0.1350 & 0.1375 & 0.1484 & 0.5197 & & Ours-Mono & 0.1011 & 0.1092 & 0.0707 & 0.0916 \\
& \textbf{Ours} & 0.0780 & 0.1281 & 0.1003 & 0.3012 & & \textbf{Ours} & \textbf{0.0780} & \textbf{0.0474} & \textbf{0.0199} & \textbf{0.0598} \\
\midrule
EMAIL & DiGress & 0.2192 & 0.6012 & 0.0702 & 0.3416 & RT & DiGress & 0.4164 & \textbf{0.1327} & 0.4147 & 0.5957 \\
& LGGM-X & 0.0751 & 0.2364 & 0.0768 & 0.3089 & & LGGM-X & \textbf{0.0525} & 0.1429 & \textbf{0.1330} & 0.2219 \\
& Ours-Mono & 0.1396 & 0.0921 & 0.0600 & 0.0765 & & Ours-Mono & 0.2911 & 0.4537 & 0.3705 & 0.3612 \\
& \textbf{Ours} & \textbf{0.0631} & \textbf{0.0529} & \textbf{0.0563} & \textbf{0.0725} & & \textbf{Ours} & 0.1999 & 0.4070 & 0.2345 & 0.2452 \\
\midrule
WEB & DiGress & 0.2556 & 0.6186 & 0.1877 & 0.6045 & COL & DiGress & 0.2473 & 0.5826 & 0.2314 & 0.7679 \\
& LGGM-X & \textbf{0.0648} & 0.3961 & \textbf{0.0549} & \textbf{0.1127} & & LGGM-X & \textbf{0.0736} & 0.5769 & \textbf{0.0895} & \textbf{0.0988} \\
& Ours-Mono & 0.1553 & 0.1456 & 0.1749 & 0.3319 & & Ours-Mono & 0.2143 & 0.3022 & 0.2329 & 0.7161 \\
& \textbf{Ours} & 0.0692 & \textbf{0.0795} & 0.1145 & 0.2760 & & \textbf{Ours} & 0.1020 & \textbf{0.1683} & 0.1863 & 0.6032 \\
\midrule
ROAD & DiGress & 0.3705 & 0.8226 & 0.2801 & 0.7198 & ECO & DiGress & 0.5431 & 0.7915 & 0.2338 & 0.6045 \\
& LGGM-X & \textbf{0.0713} & 0.2193 & \textbf{0.0987} & \textbf{0.2986} & & LGGM-X & 0.4753 & 0.3904 & 0.3194 & 0.3934 \\
& Ours-Mono & 0.4300 & 0.6884 & 0.3142 & 0.7433 & & Ours-Mono & 0.1744 & 0.3025 & 0.0436 & 0.2044 \\
& \textbf{Ours} & 0.1107 & \textbf{0.1850} & 0.1446 & 0.6027 & & \textbf{Ours} & \textbf{0.0632} & \textbf{0.2298} & \textbf{0.0326} & 0.2040 \\
\midrule
POWER & DiGress & 0.3726 & 0.4582 & 0.3270 & 1.4732 & CITATION & DiGress & 0.2527 & 0.7790 & 0.1315 & 0.4966 \\
& LGGM-X & \textbf{0.0119} & \textbf{0.1293} & \textbf{0.0373} & \textbf{0.0754} & & LGGM-X & 0.1348 & 0.7257 & 0.1160 & 0.4981 \\
& Ours-Mono & 0.1963 & 0.3947 & 0.2809 & 0.9536 & & Ours-Mono & 0.1364 & 0.2654 & 0.0917 & 0.1601 \\
& \textbf{Ours} & 0.1660 & 0.2992 & 0.2265 & 0.6956 & & \textbf{Ours} & \textbf{0.0407} & \textbf{0.1150} & \textbf{0.0621} & \textbf{0.1010} \\
\midrule
ALL & DiGress & 0.3112 & 0.5622 & 0.2030 & 0.5923 & & & & & & \\
& LGGM-X & 0.1408 & 0.3422 & 0.1253 & \textbf{0.2616} & & & & & & \\
& Ours-Mono & 0.1950 & 0.2785 & 0.1584 & 0.3632 & & & & & & \\
& \textbf{Ours} & \textbf{0.0895} & \textbf{0.1519} & \textbf{0.1057} & 0.2759 & & & & & & \\
\bottomrule
\end{tabular}
}
\par\smallskip
\begin{minipage}{\linewidth}\footnotesize\raggedright
DEG, CC, Spec, Orb: MMD of degree, clustering coefficient, Laplacian eigenvalues and orbit counts. Ours and Ours-Mono are means over 10 seeds. Full table with SDs in Table~\ref{tab:lggm_full}). On ECO, the orbit kernel saturates, giving the same value for any set of distinct dense samples. This cell is not used to mark the best method.
\end{minipage}
\vspace{-6.5mm}
\end{table}
\subsection{Zero-Shot Transfer to Empirical Data}

\subsubsection{Graph Generation Benchmark}
\label{sec:benchmarks}
The preceding sections show that the model spans diverse topologies (Sec.~\ref{sec:expressivity}) and that its parameters are recovered accurately in graph sizes not seen during training (Sec.~\ref{sec:recovery}). We now test whether these combine to transfer to real-world graphs, using a model trained only on synthetic graphon realizations.  We use the zero-shot graph-generation benchmark and baselines from \citet{wang2025large}. LGGM-X is pretrained on over $5{,}000$ empirical graphs, with the target domain $X$ held out. We use the same code, test sets, and metrics, and report baseline results from \citet{wang2025large}. The complete protocol is in Appendix~\ref{app:lggm_protocol}. As an ablation, \textbf{Ours-Mono} is trained on graphons with $p^{(1)} \equiv 0.5$, removing the topological diversity of the prior (Sec.~\ref{sec:expressivity}).

Our model achieves the best average performance on three of four metrics (``ALL'' in Table~\ref{tab:twocolumn_results}) and wins in all metrics in 4/12 domains. These results are notable given the contrasting training paradigms. LGGM-X is trained on thousands of real-world networks across many domains, whereas our model is trained \textit{exclusively} on synthetic realizations of the multifractal step graphon. In addition, our method uses $\sim$22$\times$ fewer parameters ($\sim$205K vs.\ 4.61M). LGGM-X achieves stronger performance on some domains, such as infrastructure networks, where our method may struggle as it was trained on synthetic graphs without geometric priors. Averaged across domains, the ablation increases MMD on all four metrics, indicating that the topological diversity of the prior contributes to the transfer. Taken together, the results suggest that the generative rule and the inference of its parameters transfer to real-world networks in various domains.

\subsubsection{Applications to Single-Observation Empirical Data}
\label{sec:emp_data}
Recent advances in connectomics have produced datasets with unprecedented scale, such as the complete connectome of \textit{Drosophila} \citep{dorkenwald2024neuronal}. However, these data often exist as a single observation. Whole-graph generation on these datasets has been inaccessible to existing deep generative models since they require training corpora, and foundation graph models are still in development \citep{wang2025graph}. In order to bridge this gap, we assess our model's ability to infer meaningful structures in such datasets, with sizes ranging from $N = 163$ to $N = 6301$.

\textbf{Evaluation for single-graph datasets.} Standard graph generation evaluation uses MMD between distributions of graph statistics computed across many graphs in test and generated sets~\citep{you2018graphrnn, vignac2022digress, wang2025large}. This metric is ill-defined at $N=1$ since no test distribution exists with a single empirical graph. We instead compare intra-graph distributions between empirical and synthetic networks using 1-Wasserstein distances on degree distributions and Laplacian spectra, as well as mean squared error on the multifractal spectra (Appendix ~\ref{app:metrics_definitions}).

\textbf{Reference comparison with per-instance fitting.}  We calibrate against WMGM~\citep{yang2021hidden}, which is to our knowledge the only generative-rule estimator for self-similar networks with weighted edges and node heterogeneity. WMGM optimizes on each network separately via EM. WMGM thus serves as a reference for what direct fitting achieves, not as a competing baseline. Our zero-shot method is better than or tied with WMGM on 10 of 18 comparisons (2 better, 8 tied, 8 worse), while running 9--550$\times$ faster (Table~\ref{tab:wmgm_compare}). Differences are largest in spectral error, where WMGM is better on four of six networks. This indicates that our zero-shot method reconstructs the multifractal and degree structure of real networks comparably to direct fitting, with a remaining gap in spectral structure.

\begin{table*}[t]
\centering
\caption{Reconstruction of Single-Instance Networks}
\label{tab:wmgm_compare}
\resizebox{\textwidth}{!}{%
\begin{tabular}{llcccc|llcccc}
\toprule
Domain & Method & Time (ms) & MFrac & DegW & SpecW & Domain & Method & Time (ms) & MFrac & DegW & SpecW \\
\midrule
C. elegans & WMGM & 2060 & 2.68e-2 & \textbf{5.67e-4} & \textbf{6.97e-2} & P2P Gnutella & WMGM & 2540 & 5.19e-3 & 1.41e-4 & \textbf{8.09e-2} \\
 & Ours & \textbf{3.76} & \textbf{2.19e-4} & 1.20e-3 & 1.27e-1 &  & Ours & \textbf{198} & 3.56e-3 & 1.23e-4 & 1.24e-1 \\
\midrule
Drosophila & WMGM & 1930 & 7.35e-3 & 2.29e-4 & \textbf{6.41e-2} & Power Grid & WMGM & 1810 & 5.88e-4 & 4.23e-5 & \textbf{1.81e-1} \\
 & Ours & \textbf{14.1} & 3.84e-3 & 2.69e-4 & 1.18e-1 &  & Ours & \textbf{198} & 1.24e-3 & \textbf{2.54e-5} & 2.58e-1 \\
\midrule
Allen Mouse & WMGM & 1050 & \textbf{1.52e-4} & 1.38e-4 & 3.99e-3 & MN Roads & WMGM & 1490 & \textbf{1.73e-3} & \textbf{5.92e-5} & 1.86e-1 \\
 & Ours & \textbf{2.95} & 7.86e-3 & 9.86e-5 & 6.82e-3 &  & Ours & \textbf{50.9} & 2.42e-3 & 1.24e-4 & 1.84e-1 \\
\bottomrule
\end{tabular}
} % close resizebox
\par\smallskip
\begin{minipage}{\textwidth}\footnotesize\raggedright
MFrac: multifractal MSE; DegW: degree Wasserstein; SpecW: spectral Wasserstein. Error metrics are means over 10 seeds, with the full table with SD values reported in Table~\ref{tab:wmgm_full}. Times are defined in Appendices~\ref{app:training} and~\ref{app:single_instance}. Bold marks the lower value where the two methods' mean $\pm$ SD intervals do not overlap.
\vspace{-2mm}
\end{minipage}
\end{table*}

\begin{table}[t]
\centering
\caption{Revealing Hidden Differences in EEG Networks Under Sedation.}
\label{tab:eeg_alpha_groups}
\small
\setlength{\tabcolsep}{4pt}
\resizebox{\linewidth}{!}{%
\begin{tabular}{lcccc cccc}
\toprule
Measure & Baseline & Mild & Moderate & Recovery & $d$ Base\,$-$\,Mod & $d$ Mild\,$-$\,Mod & $d$ Rec\,$-$\,Mod & $d$ Base\,$-$\,Rec \\
\midrule
\multicolumn{9}{l}{\textit{Inferred generative rule $\hat\theta^{(1)}$}} \\
Assortativity index & $-0.75 \pm 0.32$ & $-0.74 \pm 0.27$ & $-0.48 \pm 0.34$ & $-0.76 \pm 0.29$ & $\mathbf{-0.84}$ (0.043) & $\mathbf{-0.87}$ (0.043) & $\mathbf{-0.91}$ (0.036) & $0.03$ \\
$\Delta\alpha$ & $3.45 \pm 2.29$ & $3.44 \pm 2.11$ & $1.71 \pm 1.86$ & $3.55 \pm 2.11$ & $\mathbf{0.83}$ (0.048) & $\mathbf{0.87}$ (0.019) & $\mathbf{0.92}$ (0.019) & $-0.05$ \\
\midrule
\multicolumn{9}{l}{\textit{Classical network metrics}} \\
Mean clustering & $0.67 \pm 0.12$ & $0.66 \pm 0.12$ & $0.73 \pm 0.16$ & $0.68 \pm 0.11$ & $-0.40$ & $-0.45$ & $-0.36$ & $-0.07$ \\
Mean shortest path & $1.36 \pm 0.16$ & $1.37 \pm 0.16$ & $1.28 \pm 0.16$ & $1.36 \pm 0.15$ & $0.48$ & $0.57$ & $0.47$ & $0.03$ \\
Modularity & $0.053 \pm 0.045$ & $0.052 \pm 0.030$ & $0.041 \pm 0.026$ & $0.046 \pm 0.031$ & $0.32$ & $0.39$ & $0.19$ & $0.16$ \\
Spectral gap & $0.81 \pm 0.12$ & $0.81 \pm 0.10$ & $0.84 \pm 0.08$ & $0.82 \pm 0.10$ & $-0.31$ & $-0.43$ & $-0.23$ & $-0.10$ \\
\bottomrule
\end{tabular}}
\par\smallskip
\begin{minipage}{\linewidth}\footnotesize\raggedright
Values are mean $\pm$ SD over $n=20$ sessions per state; the same 20 subjects contribute to every state. $d$ is Cohen\textquotesingle s $d$ (pooled SD) for the first state minus the second. Bold marks values whose two-sided Mann\textendash Whitney $U$ test survives Holm correction over the four contrasts ($p<0.05$). The corrected $p$ is shown in parentheses only for significant values.
\end{minipage}
\vspace{-3 mm}
\end{table}

\textbf{Inferred parameters track brain state.} To demonstrate the utility of our approach in scientific applications, we next ask whether the inferred parameters reflect reversible changes in network organization. We analyze alpha-band EEG connectivity from 20 subjects at baseline, under mild and moderate propofol sedation, and after recovery \citep{chennu2016brain} (Appendix~\ref{app:eeg}). We focus on the alpha band (8--15\,Hz) functional network, where propofol produces its most prominent EEG changes on this frequency band \citep{purdon2013electroencephalogram}.

From $\hat{p}^{(1)}$ we derive two summaries, the assortativity index and multifractal width ($\Delta\alpha$) (Appendix~\ref{app:eeg}). Both metrics separate moderate sedation from every other state (Cohen's $|d| = 0.8$--$0.9$, the difference in means divided by the pooled SD, Holm-corrected $p < 0.05$), while values at baseline, mild sedation and recovery are similar (Table~\ref{tab:eeg_alpha_groups}). Under moderate sedation the rule becomes less disassortative and its multifractal spectrum narrows ($\Delta\alpha$: $3.4 \to 1.7$), consistent with reported changes in multiscale complexity \citep{casali2013theoretically, wilsenach2025graph} and integration \citep{luppi2021brain} under altered consciousness. Critically, both metrics return to levels similar to baseline on recovery, suggesting that they reflect a reversible, state-dependent change rather than drift over the recording session. In contrast, traditional network statistics show smaller effects ($|d| = 0.2$--$0.6$), none significant after correction. The two summaries also decode moderate sedation in unseen subjects (AUC $0.69$--$0.75$), suggesting that they carry state information that is shared between subjects (Fig.\ref{fig:eeg_supp}A). Moreover, the fit does not differ across states, indicating that the difference in inferred parameters during moderate sedation is not confounded by changes in the model fit (Fig.\ref{fig:eeg_supp}B). Lastly, 76 of 80 networks fit within the error range expected when the model is true (Fig.\ref{fig:eeg_supp}C).

\vspace{-3 mm}

\section{Discussion}

\vspace{-3mm}
Here we reiterate that comparisons with LGGM-X and WMGM serve as reference points for transfer and fit. As such, we do not claim that our formulation is intended to supersede deep or mechanistic models in general. Rather, it is a complementary method suited to settings requiring both interpretable parameters and amortized inference, such as in scientific applications when N=1 and networks are of sizes unseen during training. We also do not claim that MWNOT is the optimal inference method for this framework. Our approach has several limitations. For one, the model assumes the same connection rule repeats across scales. Thus, networks whose wiring rule varies irregularly across scales fall outside this assumption. In addition, the inverse operator is not permutation-equivariant. We resolve this with a deterministic degree-sorting canonicalization, which is always available without any metadata (Appendix~\ref{app:ordering}). Lastly, the model needs improvement in the spectral metrics. Despite these limitations, the results presented here indicate that mechanistic interpretability and amortized inference can be effectively combined within a unified framework to improve our understanding of complex networks. Future work includes making the generative rule more expressive and comparing other inverse-operator architectures to identify which yields the most accurate inference.

\clearpage
\section*{AI Use Statement}
In this work, we used generative AI tools to design and provide feedback on research methodology and experiments, implement methods, and interpret results. We have not used generative AI tools to help develop theoretical models or conceptual frameworks, formulate mathematical claims, provide critical ingredients for proving mathematical claims, assist in the writing of proofs, propose or refine hypotheses, or clean and reformat datasets, and generating synthetic datasets, assisting with translation, and supporting qualitative and thematic data analysis are not applicable to this work. Additionally, we used generative AI tools to suggest experimental parameters, create and edit software code, create scientific figures, draft parts of the paper, edit the paper to improve readability, identify relevant literature, and format references. We have reviewed all AI-assisted work. All AI-generated code, analyses and text edits were reviewed by the authors. We take responsibility for the final content of this work, including text, claims or artifacts produced with the aid of generative AI.

\section*{Reproducibility Statement}
We describe the full experimental protocol in the Appendix,
including preprocessing, training settings,
hyperparameters, and evaluation procedures. Anonymized source code to
reproduce all experiments is provided in the supplementary materials,
with instructions for running each experiment.

\clearpage
\FloatBarrier
% --- References Section ---
\bibliographystyle{iclr2027_conference}
\bibliography{references}

% --- Start of Supplemental Material ---
\clearpage
\appendix

% Reset figure and table numbering for the appendix
\renewcommand{\thefigure}{S\arabic{figure}}
\setcounter{figure}{0}
\renewcommand{\thetable}{S\arabic{table}}
\setcounter{table}{0}
\renewcommand{\theequation}{S\arabic{equation}}
\setcounter{equation}{0}

\begin{center}
    \Large\textbf{Supplemental Material}
\end{center}

% closed form
\vspace{1em}
\section{Model Selection}
\label{app:choice_M}

\paragraph{Selecting $M$.} $M$ is the number of base communities. The base graphon splits the latent space $[0,1]$ into $M$ intervals and each recursion step splits every interval into $M$ sub-intervals. The depth-$K$ graphon therefore partitions the latent space into $M^K$ blocks, and each block must contain enough nodes for its parameters to be reflected in the observed graph. For our smallest training graphs at the largest depth ($N=512$, $K=6$; training pairs $K=6$ only with $N \ge 512$), $M=2$ yields $2^6=64$ blocks, whereas $M=4$ would yield $4^6=4096$ blocks, far exceeding the number of nodes and leaving most blocks empty. Supporting larger $M$ would therefore require exponentially larger training graphs.

\paragraph{Expressivity.} Fixing $M=2$ does not restrict the model to two communities. Expanding two recursive steps, $l^{(1)}\otimes l^{(1)}$ and $p^{(1)}\otimes p^{(1)}$ form a valid $M=4$ base rule, so an $M=2$ model at depth $K$ is equivalent to the subset of the $M=4$ model class at depth $K/2$ whose base rule is Kronecker-factorizable. What an unrestricted $M=4$ model adds is capacity for non-factorizable branching rules. Periodic multirule generators (Appendix~\ref{app:compression}) offer intermediate expressivity without requiring larger training graphs, and are a natural direction for future work.

\section{Closed-Form Network Statistics under the Multifractal Step Graphon}
\label{app:closed-form}

A consequence of the Kronecker structure of $W_K$ is that several network statistics admit closed-form expressions that factorize across scales. We derive two such expressions here: the expected edge density and the expected degree of a node assigned to a given block.

\subsection{Expected edge density}

The expected edge density of $W_K$ is
\begin{equation}
\rho(W_K) = \int_0^1 \int_0^1 W_K(x, y) \, dx \, dy = \sum_{q, h} l^{(K)}_q \, l^{(K)}_h \, p^{(K)}_{qh}.
\end{equation}
Because $l^{(K)}$ and $p^{(K)}$ each factorize as $K$-fold Kronecker products, the sum decomposes into a product of $K$ identical scale-1 sums:
\begin{equation}
\rho(W_K) = \prod_{k=1}^{K} \left( \sum_{i, j} l^{(1)}_i \, l^{(1)}_j \, p^{(1)}_{ij} \right) = \rho(W_1)^K.
\end{equation}
Edge density at depth $K$ is therefore the $K$-th power of the depth-1 density, making explicit the recursive nature of the construction.

\subsection{Expected degree of a node in block $q$}

For a graph of $N$ nodes with latent coordinates drawn iid from $\text{Uniform}(0,1)$, the expected out-degree of a node assigned to block $q = (q_1, \ldots, q_K)$ is
\begin{equation}
\mathbb{E}[d^{\text{out}}_q] = N \sum_h l^{(K)}_h \, p^{(K)}_{qh} = N \prod_{k=1}^{K} \left( \sum_{j} l^{(1)}_{j} \, p^{(1)}_{q_k j} \right).
\end{equation}
The expression for in-degree is analogous, with $p^{(1)}_{q_k j}$ replaced by $p^{(1)}_{j q_k}$. Each scale-$k$ factor depends only on $q_k$, so the expected degree of a node decomposes into independent contributions from each level of the recursion.

\subsection{Multifractal spectrum}

The multifractal structure of $W_K$ admits a closed-form characterization through the partition function. Define the scale-1 moment sum
\begin{equation}
Z(q) = \sum_{i, j} l^{(1)}_i \, l^{(1)}_j \, \left(p^{(1)}_{ij}\right)^q.
\end{equation}
Because the multiplicative structure of $W_K$ makes the depth-$K$ partition function factorize as $Z_K(q) = Z(q)^K$, the mass exponent $\tau(q)$ takes the closed form
\begin{equation}
\tau(q) = -\frac{\log Z(q)}{\log M}.
\end{equation}
The Lipschitz--H\"older exponent $\alpha(q)$ and multifractal spectrum $f(\alpha)$ are then obtained via the Legendre transform:
\begin{equation}
\alpha(q) = \frac{d\tau}{dq}, \qquad f(\alpha(q)) = q\, \alpha(q) - \tau(q).
\end{equation}
The full multifractal spectrum is therefore determined by $\theta^{(1)}$ alone, independent of $K$ in the limit. This closed form makes precise the claim in Section~\ref{sec:multifractal-extension} that $\theta^{(1)}$ controls where $W_K$ lies on the monofractal--multifractal continuum: monofractal behavior corresponds to $Z(q) = c^q$ for some constant $c$, in which case $\tau(q)$ is linear and $f(\alpha)$ collapses to a single point; heterogeneity in $p^{(1)}$ produces curvature in $\tau(q)$ and a non-degenerate spectrum.

\subsection{Implication}

These factorizations show that macroscopic network statistics inherit the recursive structure of the generator: depth-$K$ observables are powers or products of depth-1 quantities, and the multifractal spectrum is fully characterized by $\theta^{(1)}$ alone. This provides analytical tools for relating $\theta^{(1)}$ to observable network properties, and confirms that the recursive construction is not merely a notational device but produces statistics that scale predictably with depth.
% ==========================================
% Section B: Identifiability Proof
% ==========================================
\section{Proof of Proposition 1 (Identifiability)}
\label{app:identifiability}

We establish identifiability through three steps: (i) the step graphon $W_K$ determines its block structure uniquely up to permutation, (ii) the Kronecker factorization of the block parameters is unique, and (iii) the combined result yields identifiability of the base parameters up to community label permutation.

\subsection{Step 1: A step graphon determines its block parameters up to permutation}

A step graphon $W_K$ is a piecewise-constant function on $[0,1]^2$ defined by a partition of $[0,1]$ into $M^K$ contiguous intervals with lengths $l^{(K)} \in \mathrm{int}(\Delta^{M^K-1})$ and a block interaction matrix $p^{(K)} \in (0,1)^{M^K \times M^K}$. We show that these parameters are determined by $W_K$ up to a simultaneous permutation of the blocks.

Since all entries of $l^{(K)}$ are strictly positive, the partition consists of $M^K$ intervals each with positive Lebesgue measure. The function $W_K$ takes the constant value $p^{(K)}_{qh}$ on the rectangle $I_q \times I_h$.

Suppose two parameterizations $(l^{(K)}_1, p^{(K)}_1)$ and $(l^{(K)}_2, p^{(K)}_2)$ yield the same graphon $W_K$ almost everywhere. Because the graphon is piecewise constant with finitely many pieces, equality a.e.\ implies equality on every rectangle of positive measure.

The partition boundaries are determined by the discontinuities of $W_K$. For a fixed $y$ in the interior of some interval $I_h$, the function $W_K(\cdot, y)$ is piecewise constant in $x$, taking the value $p^{(K)}_{qh}$ when $x \in I_q$. The discontinuities of this function occur at the boundaries between intervals $I_q$ and $I_{q+1}$ whenever $p^{(K)}_{qh} \neq p^{(K)}_{(q+1)h}$. By varying $y$ across all intervals, a boundary between $I_q$ and $I_{q+1}$ is undetectable only if $p^{(K)}_{qh} = p^{(K)}_{(q+1)h}$ for every $h$ --- that is, blocks $q$ and $q+1$ have identical interaction profiles with all other blocks.

In this case, blocks $q$ and $q+1$ are indistinguishable: they have the same size (since both parameterizations must agree on the set of interval widths) and the same interaction profile. Swapping their labels produces the same graphon. This is precisely a permutation of the block indices.

Because the intervals are contiguous (ordered left-to-right on $[0,1]$), the set of partition boundaries is a property of the function itself, not of the parameterization. Both parameterizations must use the same boundary points and therefore the same set of interval widths. The only freedom is in which interval is assigned which label. Thus, there exists a permutation $\pi \in S_{M^K}$ such that:
\begin{equation}
    l^{(K)}_{1,q} = l^{(K)}_{2,\pi(q)}, \quad p^{(K)}_{1,qh} = p^{(K)}_{2,\pi(q)\pi(h)}
\end{equation}
for all $q, h \in \{1, \dots, M^K\}$. \qed

\subsection{Step 2: Uniqueness of the Kronecker factorization}

We now show that if two base parameterizations $(l_1, p_1)$ and $(l_2, p_2)$ yield the same expanded parameters under $K$-fold Kronecker product (up to a permutation $\pi \in S_{M^K}$), then they are related by a base-level permutation $\sigma \in S_M$.

\subsubsection{Step 2a: Factorization of the length vector}

By construction, $l^{(K)} = l^{(1)} \otimes \cdots \otimes l^{(1)}$ ($K$ times). The entries of $l^{(K)}$ are indexed by multi-indices $q = (q_1, \dots, q_K) \in \{1, \dots, M\}^K$ and take the form:
\begin{equation}
    l^{(K)}_q = \prod_{k=1}^K l^{(1)}_{q_k}
\end{equation}
These are the entries of a rank-1 symmetric tensor $l^{(1)} \otimes \cdots \otimes l^{(1)} \in \mathbb{R}^{M \times \cdots \times M}$. We show that $l^{(1)}$ is recoverable from this tensor up to permutation of its entries.

Consider the $M \times M^{K-1}$ matricization of $l^{(K)}$, denoted $L_{(1)}$, obtained by unfolding along the first mode:
\begin{equation}
    L_{(1)} = l^{(1)} \cdot \left( l^{(1)} \otimes \cdots \otimes l^{(1)} \right)^\top
\end{equation}
Since all entries of $l^{(1)}$ are strictly positive, $L_{(1)}$ has rank 1 and its column space is spanned by $l^{(1)}$. The simplex constraint $\sum_i l^{(1)}_i = 1$ fixes the scale uniquely. The only remaining ambiguity is permutation of the entries of $l^{(1)}$.

To see why the permutation must act consistently across all modes: a permutation $\pi \in S_{M^K}$ of the multi-indices that maps one Kronecker product to another must act independently on each coordinate of the multi-index. If $\pi$ maps $(q_1, \dots, q_K)$ such that the first coordinate $q_1$ is permuted by $\sigma$ in one multi-index but left unchanged in another, the multiplicative structure $l^{(K)}_q = \prod_k l^{(1)}_{q_k}$ would be violated. Formally, the logarithms of the entries form an additive rank-1 tensor:
\begin{equation}
    \log l^{(K)}_q = \sum_{k=1}^K \log l^{(1)}_{q_k}
\end{equation}
A permutation preserving this additive structure must act as the same base permutation $\sigma$ on each coordinate independently \citep{kolda2009tensor}.

Therefore, if $l_1^{(K)} = \pi(l_2^{(K)})$, then $\pi$ is induced by a base permutation $\sigma \in S_M$ acting as $\sigma(q_1, \dots, q_K) = (\sigma(q_1), \dots, \sigma(q_K))$, and $l_1 = \sigma(l_2)$.

\subsubsection{Step 2b: Factorization of the interaction matrix}

Similarly, $p^{(K)} = p^{(1)} \otimes \cdots \otimes p^{(1)}$ ($K$ times). The entries satisfy:
\begin{equation}
    p^{(K)}_{qh} = \prod_{k=1}^K p^{(1)}_{q_k h_k}
\end{equation}
Let $\mathrm{vec}(p^{(1)}) \in (0,1)^{M^2}$ denote the vectorization of $p^{(1)}$. Then:
\begin{equation}
    \mathrm{vec}(p^{(K)}) = \mathrm{vec}(p^{(1)}) \otimes \cdots \otimes \mathrm{vec}(p^{(1)})
\end{equation}
By the same rank-1 tensor uniqueness argument as in Step 2a, $\mathrm{vec}(p^{(1)})$ is recoverable up to permutation. However, this permutation is not arbitrary over $M^2$ entries --- it must be consistent with the base-level permutation $\sigma$ established in Step~2a.

The permutation $\pi$ acts simultaneously on both $l^{(K)}$ and $p^{(K)}$, because both are defined over the same block indices. Block $q$ has size $l^{(K)}_q$ and interaction values $p^{(K)}_{qh}$. Any permutation that preserves both must be factorizable as $\sigma$ acting independently on each level of the multi-index. The permutation $\sigma$ from Step~2a induces the action:
\begin{equation}
    p_{1,ij} = p_{2,\sigma(i)\sigma(j)}
\end{equation}
Therefore $p_1 = \sigma(p_2)$ under the same permutation $\sigma$. \qed

\subsection{Step 3: Combining the results}

From Step~1, $W_K^{(1)} = W_K^{(2)}$ a.e.\ implies that $(l_1^{(K)}, p_1^{(K)})$ and $(l_2^{(K)}, p_2^{(K)})$ agree up to a block permutation $\pi \in S_{M^K}$. From Step~2, the Kronecker structure forces $\pi$ to be induced by a base-level permutation $\sigma \in S_M$, yielding $l_1 = \sigma(l_2)$ and $p_1 = \sigma(p_2)$.

This establishes that the multifractal step graphon $W_K$ uniquely determines the base parameters $\theta^{(1)} = (l^{(1)}, p^{(1)})$ up to community label permutation.

\subsection{Discussion of Assumptions}

\textbf{Strict positivity of $l^{(1)}$.} The requirement $l^{(1)} \in \mathrm{int}(\Delta^{M-1})$ ensures that every community has positive measure and contributes to the graphon. If some $l^{(1)}_i = 0$, the corresponding community is absent and the effective model has fewer than $M$ communities, reducing to an identifiable model of smaller dimension.

\textbf{Strict positivity of $p^{(1)}$.} The requirement $p^{(1)}_{ij} \in (0,1)$ ensures that the vectorized interaction matrix has no zero entries, which is needed for the rank-1 tensor factorization to be unique. If $p^{(1)}_{ij} = 0$ for some $(i,j)$, multiple factorizations may produce the same zero entry. In practice, exact zeros are measure-zero in the parameter space and do not arise in the sampling procedure described in Appendix~\ref{app:training}.

\textbf{$K \geq 2$.} The result requires at least two Kronecker expansions. For $K = 1$, the graphon is a standard step graphon with $M$ blocks, and identifiability reduces to the classical step graphon case (covered by Step~1 alone).

Note this proposition does not claim identifiability over arbitrary graphon representatives under all measure-preserving transformations, nor does it imply statistical stability from finite noisy graph samples.
\vspace{2em}

% ==========================================
% Section C: architecture
% ==========================================
\section{Multiwavelet Neural Operator Transformer}
\label{app:architecture}

\subsection{Architecture}

Given an observed adjacency matrix $A \in \mathbb{R}^{N \times N}$, the model processes the graph topology through the following operations.

\textbf{Channel lift.}
A learned linear map lifts each entry of the input $A$ to $k^2$ channels, reshaping it into $x^{(0)} \in \mathbb{R}^{N \times N \times k^2}$. The Legendre polynomial basis enters through the fixed multiwavelet filters.

\textbf{Multiwavelet decomposition.}
A 2D discrete multiwavelet transform is applied over $L$ spatial scales using a non-standard wavelet representation \citep{gupta2021multiwavelet}. The multiwavelet basis is constructed via Gram--Schmidt orthogonalization of shifted Legendre polynomials, ensuring that high-frequency details are strictly orthogonal to the coarse approximation. At each scale $\ell \in \{1, \dots, L\}$, the feature grid is decomposed by a coarse filter $\mathcal{H}$ and a composite detail filter $\mathcal{G}$:
\begin{align}
    s^{(\ell)} &= x^{(\ell-1)} * \mathcal{H} \label{eq:wavelet_s} \\
    d^{(\ell)} &= x^{(\ell-1)} * \mathcal{G} \label{eq:wavelet_d}
\end{align}
The coarse coefficients $s^{(\ell)}$ capture macroscopic connectivity, while $d^{(\ell)}$ concatenates horizontal (LH), vertical (HL), and diagonal (HH) high-frequency variations. The output $x^{(\ell)} = s^{(\ell)}$ serves as input for the next coarser scale. This hierarchical decomposition mirrors the multiplicative construction of the multifractal step graphon. When $N$ is not a power of two, the canonicalised input (Appendix~\ref{app:ordering}) is cropped to the largest dyadic size $\leq N$, which drops the lowest-degree nodes.

\textbf{Multiscale tokenization.}
The wavelet coefficients (details $d^{(\ell)}$ for $\ell \le L$ and the final coarse $s^{(L)}$) are flattened into sequences and concatenated with positional grids and scale-indicator labels to form multiscale tokens $T^{(\ell)}$.

\textbf{Cross-attention for operator approximation.}
The model parameterizes the inverse operator via cross-attention in a branch-trunk architecture \citep{lu2021learning, shih2025transformers, hao2023gnot}. The multiscale tokens encode the input function (branch), while learnable query embeddings evaluate the target parameters (trunk).

Let $Q_l, Q_p \in \mathbb{R}^{M \times D}$ denote query matrices for $l^{(1)}$ and $p^{(1)}$. At each scale $\ell$, tokens $T^{(\ell)}$ are projected into Keys $K^{(\ell)} \in \mathbb{R}^{N_\ell \times D}$ and Values $V^{(\ell)} \in \mathbb{R}^{N_\ell \times S}$. The cross-attention aggregates these into fixed-size latent representations:
\begin{equation} \label{eq:cross_attention}
    z_s^{(\ell)}(m) = \sum_{d=1}^D \left( \frac{1}{N_\ell} \sum_{j=1}^{N_\ell} k_d \left(w_j^{(\ell)}\right)^T v_s \left(w_j^{(\ell)}\right) \right) q_d(m)
\end{equation}
where $m \in \{1, \dots, M\}$ indexes the target community and $w_j^{(\ell)}$ is the $j$-th wavelet token within $T^{(\ell)}$. The inner summation averages over $N_\ell$ grid locations, forming a data-dependent basis from the wavelet coefficients. The query embedding $q_d(m)$ extracts the relevant features, yielding attention outputs $Z_l^{(\ell)}, Z_p^{(\ell)} \in \mathbb{R}^{M \times S}$.

\textbf{Decoding.}
Scale-specific representations are concatenated into global embeddings $Z_p = [Z_p^{(1)}, \dots, Z_p^{(L)}]$ and $Z_l = [Z_l^{(1)}, \dots, Z_l^{(L)}]$. MLPs decode these to predict the base parameters:
\begin{align}
    \hat{p}^{(1)} &= \text{MLP}_p(Z_p) \in [0, 1]^{M \times M} \label{eq:mlp_p} \\
    \hat{l}^{(1)} &= \text{MLP}_l(Z_l) \in \Delta^{M-1} \label{eq:mlp_l}
\end{align}

\section{Empirical Validation of the Multiwavelet Basis}
\label{app:wavelet-validation}

We empirically verify that the multiwavelet basis used in MWNOT is responsive to the recursive multiscale construction of the multifractal step graphon to support the inductive bias of using this basis for multifractal objects.

\paragraph{Method.} For each recursion depth $K \in \{1, \ldots, 6\}$, we construct the continuous graphon $W_K$ on an $N \times N$ grid ($N = 1024$) using fixed base parameters $l^{(1)} = [0.4, 0.6]$ and $p^{(1)} = \begin{bmatrix} 0.9 & 0.1 \\ 0.1 & 0.9 \end{bmatrix}$ for demonstration. We project $W_K$ onto the constant component of the Legendre multiwavelet basis and recursively apply the 2D multiwavelet transform $L$ times. At each scale $\ell$, we record the squared-norm energy of the detail coefficients $E^{(\ell)} = \sum |d^{(\ell)}|^2$, and after the final iteration, we record the energy of the coarse approximation $E^{(L)}_{\text{coarse}} = \sum |s^{(L)}|^2$. From these we compute two summary quantities:
\begin{align}
\rho_{\text{fine}}(K) &= \frac{E^{(1)} + E^{(2)}}{E^{(L)}_{\text{coarse}} + \sum_{\ell=1}^{L} E^{(\ell)}}, \\
\rho_{\text{detail}}(K) &= \frac{\sum_{\ell=1}^{L} E^{(\ell)}}{E^{(L)}_{\text{coarse}} + \sum_{\ell=1}^{L} E^{(\ell)}}.
\end{align}

\paragraph{Result.} Figure~\ref{fig:wavelet-validation} reports both quantities as a function of $K$. Both $\rho_{\text{fine}}$ and $\rho_{\text{detail}}$ grow approximately exponentially with $K$ (linear on a semi-log scale), spanning roughly two orders of magnitude over the recursion depths used in training.

\paragraph{Interpretation.} Each recursion level introduces finer-scale block boundaries to $W_K$, which the multiwavelet basis encodes through its detail coefficients. Deeper recursion therefore produces an exponentially larger share of energy in the detail components. MWNOT's encoder and cross-attention operate on these detail coefficients; the architecture's representational capacity is thus allocated to the components where the multifractal structure of $W_K$ resides. We interpret this as an empirical match between the basis and the construction, supporting the multiwavelet basis as a principled choice for representing this object class.

\begin{figure}[h]
    \centering
    \includegraphics[width=0.85\linewidth]{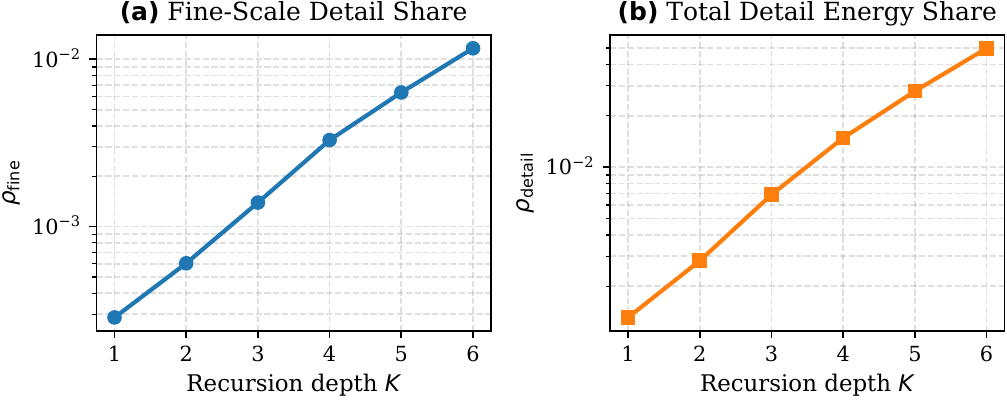}
    \caption{\textbf{Multiwavelet energy response to recursion depth.} (a) Fine-scale detail share $\rho_{\text{fine}}$. (b) Total detail energy share $\rho_{\text{detail}}$. Both quantities grow approximately exponentially with $K$, consistent with the multiplicative scaling structure of the multifractal step graphon.}
    \label{fig:wavelet-validation}
\end{figure}

\subsection{Input Ordering}
\label{app:ordering}
Function-space methods operate on inputs aligned to spatial axes. More specifically, the multiwavelet basis decomposes the adjacency matrix as a discretized function on $[0,1]^2$, which benefits from inputs whose block structure is aligned with the spatial axes. As such, function-space formulations trade architectural permutation equivariance for a parameterization that allows for generalization to unseen graph sizes at inference.

This consideration concerns only the finite-sample to graphon inverse map learned by the architecture. It is distinct from the population-level identifiability of $W_k \mapsto \theta^{(1)}$ established in Proposition~\ref{prop:identifiability}, which holds at the level of graphons, and from the exchangeability of finite graphs sampled from a graphon under the Aldous framework, which guarantees that \emph{generation} of finite samples under our model remains permutation-invariant.

\textbf{Canonicalization.} We resolve the ordering dependence with a single deterministic, parameter-free rule applied to every input, synthetic or empirical: nodes are sorted in descending order of degree (weighted strength for weighted inputs). Ties are broken by a stable sort in the zero-shot benchmark and arbitrarily elsewhere. The rule uses no metadata, is available for any graph, and makes the model input a function of the graph alone: two labelings of the same graph yield the same sorted adjacency matrix up to the order of equal-degree nodes, so the inferred $\hat\theta^{(1)}$ is invariant to node relabeling up to that residual tie ambiguity. Degree sorting also recovers an approximately block-aligned form for the wavelet basis, since expected degree is monotone in community membership under the model (Appendix~\ref{app:closed-form}).

\textbf{Training.} Synthetic graphs are produced by multinomial community assignment followed by geometric edge sampling (Appendix~\ref{app:training}). This randomizes which node sits in which community, producing adjacency matrices that resemble static noise rather than block-structured signals; the canonicalization above is applied to every generated training graph before augmentation (Appendix~\ref{app:training}).

\textbf{Inference.} The same canonicalization is applied to every empirical network before inference, in the zero-shot benchmark (Appendix~\ref{app:lggm_protocol}), the single-instance reconstructions (Appendix~\ref{app:single_instance}), and the EEG analysis. Dataset-native orderings, whether anatomical, temporal or spatial, are never used; all reported numbers are produced under the metadata-free ordering. The MMD statistics of the benchmark are permutation-invariant, so canonicalizing the model input leaves the reference side of every comparison unchanged.

% ==========================================
% Section D: Training Details
% ==========================================
\section{Training Details}
\label{app:training}

\subsection{Data Generation}

We generate synthetic training pairs $(A, \{p^{(1)}, l^{(1)}\})$ by constructing $K$-level parameters $l^{(K)}$ and $p^{(K)}$ via Kronecker products of randomly sampled base measures. The base parameters are sampled as follows. For the community distribution $l^{(1)} \in \Delta^{M-1}$, each entry is initialized to a minimum value of $0.1$ to ensure that every community has non-negligible measure, and the remaining probability mass $(1 - 0.1M)$ is distributed by drawing $M$ uniform random values and normalizing them to sum to the residual. The interaction matrix $p^{(1)} \in (0,1)^{M \times M}$ is sampled entry-wise from $\text{Uniform}(0.1, 1.0)$, ensuring strict positivity as required by the identifiability result (Proposition~1). Parameter sets with $|l^{(1)}_0 - l^{(1)}_1| \le 0.01$ are rejected. The targets are ordered by expected degree, so community 0 is the higher-degree community.

Nodes are assigned to communities from a multinomial distribution parameterized by $l^{(K)}$, and edge weights are drawn as $e_{ij} \sim \text{Geometric}(1 - p^{(K)}_{qh})$. Non-zero edge weights are log-binned to form the input adjacency matrix:
\begin{equation} \label{eq:log_binning}
    A_{ij} = \lfloor \log_2(e_{ij}) \rfloor + 1 \quad \text{for} \quad e_{ij} > 0
\end{equation}
Inputs are canonicalised as in Appendix~\ref{app:ordering} before log-binning.

\paragraph{Augmentation.} Each training graph is augmented with probability 0.5 by node dropout, edge dropout and spurious edges, each at a rate drawn from $\mathrm{Uniform}(0, 0.3)$. The edge masks are symmetric with zero diagonal. Augmented graphs are binarized and are not re-sorted after augmentation.

\subsection{Loss Function}

The model is trained with a composite loss. The interaction matrix is optimized via binary cross-entropy with logits ($\mathcal{L}_{\text{BCE}}$) on the raw MLP outputs $\hat{Z}_p$, while the node distribution $\hat{l}^{(1)} \in \Delta^{M-1}$ is optimized via KL divergence ($\mathcal{L}_{\text{KL}}$):
\begin{equation} \label{eq:loss}
    \mathcal{L}(\theta^{(1)}, \hat{\theta}^{(1)}) = \mathcal{L}_{\text{BCE}}(p^{(1)}, \hat{Z}_p) + \mathcal{L}_{\text{KL}}(l^{(1)} \parallel \hat{l}^{(1)})
\end{equation}

\subsection{Optimization and Hyperparameters}

The model is trained for 2000 epochs using AdamW with a base learning rate of $10^{-3}$, beginning with a 10-epoch linear warmup from $10^{-5}$ to $10^{-3}$, followed by cosine annealing to $10^{-5}$ over the remaining epochs. The batch size is 64 (32 per GPU on two GPUs).

Each epoch draws 2000 new graphs from the generator. Training graphs have $N \in \{128, 512, 1024\}$ and $K \in \{2, 4, 6\}$. Validation uses 100 fixed graphs at $K=4$ and $N=512$, and we keep the epoch with the lowest validation MAE of $p^{(1)}$ plus that of $l^{(1)}$ (epoch 1816). The MWNOT architecture uses the following hyperparameters: query/key dimension $D = 128$, latent feature/value dimension $S = 128$, $L = 4$ wavelet decomposition scales, and wavelet order $k = 6$ (corresponding to order-6 Legendre polynomial bases). The number of base communities is $M = 2$ throughout. Recovery error (Figure~\ref{fig:recovery}) is computed on 100 new, unaugmented graphs per $(N, K)$ (random seed 42 for the whole run), with the depth sweep at $N = 1024$, after matching predicted to true communities by the best of the $M!$ permutations.

\subsection{Compute resources.}
Training was performed on two NVIDIA V100 GPUs and completed in approximately
8 hours. Inference is fast enough to run on a single GPU or modern CPU: per-graph
inference time is reported in Table~\ref{tab:wmgm_compare} (``Time (ms)'' column)
as one forward pass on the cropped input, averaged over 10 repeats after one warm-up
pass, excluding normalisation, sorting and log-binning.

\section{Expressivity Sweep Details}
\label{app:expressivity_setup}
For both sweeps, we set $M = 2$, $K = 4$, and $N = 1024$, averaging over 10 realizations per setting. The heterogeneity sweep uses $\mathbf{l}^{(1)} = [0.25, 0.75]$ and interpolates $\mathbf{p}^{(1)}$ from a uniform matrix $[[0.5, 0.5], [0.5, 0.5]]$ to a heterogeneous matrix $[[0.99, 0.05], [0.05, 0.01]]$ as $w_{\mathrm{het}}$ varies from 0 to 1 over 30 steps. Panel (d) of Figure~\ref{fig:expressivity} is drawn from a separate sweep, in which $\mathbf{p}^{(1)}$ is interpolated in 12 steps towards each of three targets (left-skewed $[[0.99, 0.2], [0.2, 0.01]]$, symmetric $[[0.99, 0.01], [0.01, 0.99]]$ and right-skewed $[[0.01, 0.6], [0.6, 0.99]]$), with one realization per step and $q \in [-3.5, 3.5]$. The structure sweep uses $\mathbf{l}^{(1)} = [0.4, 0.6]$ and interpolates between a disassortative matrix $[[0.01, 0.99], [0.99, 0.01]]$ at $w_{\mathrm{struct}} = -1$, the uniform matrix at $w_{\mathrm{struct}} = 0$, and an assortative matrix $[[0.99, 0.01], [0.01, 0.99]]$ at $w_{\mathrm{struct}} = +1$ over 21 steps. In panels (b) and (c), degree assortativity is computed on the directed weighted graph, and Louvain modularity and the normalised-Laplacian spectral gap on the weighted undirected graph, with isolated nodes removed; $\mathbf{p}^{(1)}$ is clamped to $[0.01, 0.99]$. The multifractal spectra in these sweeps use $q$ grids of 51 points on $[-10, 10]$ or 40 points on $[-8, 8]$, instead of the grid of Appendix~\ref{app:metrics_definitions}. Panel (a)'s degree-distribution regimes use distinct hand-tuned $(\mathbf{l}^{(1)}, \mathbf{p}^{(1)}, K)$ configurations to target each canonical shape (full parameters in the released code).

The Lipschitz--Hölder exponent $\alpha$ measures local singularity strength: small $\alpha$ corresponds to hubs, large $\alpha$ to peripheral nodes. $f(\alpha)$ gives the fractal dimension of the set of nodes at each $\alpha$. The width $\Delta\alpha = \alpha_{\max} - \alpha_{\min}$ summarizes multiscale heterogeneity. Full definitions are in Appendix~\ref{app:metrics_definitions}.

\section{Zero-shot benchmark: evaluation protocol and full results}
\label{app:lggm_protocol}

This appendix states the protocol behind Table~\ref{tab:twocolumn_results} and reports the full table with seed-to-seed variability (Table~\ref{tab:lggm_full}). We follow the evaluation protocol of the LGGM benchmark \citep{wang2025large}, such as the test splits, metrics, and released evaluation code. The paragraphs below state how our model is applied within that protocol.

\paragraph{Data.} We use the twelve domains and the held-out test graphs of the LGGM benchmark \citep{wang2025large} processed exactly as the benchmark's evaluation pipeline (undirected, self-loops removed). For ECO, as in the released evaluation code, the reference combines the training, validation and test splits (251 graphs, six unique).

\paragraph{Inference.} The inverse operator receives the canonicalised (Appendix~\ref{app:ordering}) binary adjacency matrix, truncated to the largest dyadic size $2^{\lfloor \log_2 N \rfloor}$ (graphs with fewer than $2^{L}=16$ nodes are zero-padded to 16 nodes) as required by the multiwavelet decomposition. One forward pass returns $\hat\theta^{(1)}=(\hat l^{(1)},\hat p^{(1)})$. No per-graph optimisation is performed and $\hat\theta^{(1)}$ is used as returned.

\paragraph{Recursion depth.} The depth $K$ used to realise a graph from $\hat\theta^{(1)}$ is set per reference graph by the closed-form density relation of Appendix~\ref{app:closed-form}, $\rho_{\mathrm{emp}}=\rho(W_1)^{K}$ with $\rho(W_1)=\sum_{ij}\hat l_i \hat l_j \hat p_{ij}$, solved as $K^\star=\mathrm{round}\big(\log\rho_{\mathrm{emp}}/\log\rho(W_1)\big)$ and clamped to $[2,6]$. Here $\rho_{\mathrm{emp}}$ is the number of non-zero entries divided by $N^2$, diagonal included, on the full uncropped network; rounding is half-to-even, and $K=4$ is used if $\rho_{\mathrm{emp}}$ or $\rho(W_1)$ lies outside $(0,1)$. $K^\star$ depends only on $\hat\theta^{(1)}$ and the observed edge density. It is fixed before any graph is generated. Therefore, no evaluation metric is minimised over $K$.

\paragraph{Generation and metrics.} For each reference graph and each seed, one graph of the same size is sampled from $\hat\theta^{(1)}$ at depth $K^\star$. These graphs are directed and may contain self-loops. Following the benchmark's protocol, they are binarized, symmetrised by union ($A+A^\top>0$) and stripped of self-loops before calculating the reconstruction metrics. We repeat the sampling for ten seeds (42--51) and report mean and standard deviation. DiGress and LGGM-X entries are the values published with the benchmark. Ours-Mono is the ablation of Ours, trained with $p^{(1)}$ fixed at $0.5$ and evaluated under the identical protocol.

\begin{table}[p]
\caption{Full version of Table~\ref{tab:twocolumn_results} with seed-to-seed standard deviations. Lower is better.}
\label{tab:lggm_full}
\centering
\footnotesize
\setlength{\tabcolsep}{4pt}
\renewcommand{\arraystretch}{0.95}
\begin{tabular}{llcccc}
\toprule
\textbf{Domain} & \textbf{Method} & \textbf{DEG} & \textbf{CC} & \textbf{Spec} & \textbf{Orb} \\
\midrule
FB & DiGress & 0.3376 & 0.6298 & 0.0797 & 0.3593 \\
 & LGGM-X & 0.4723 & 0.6843 & 0.2924 & 0.7555 \\
 & Ours-Mono & 0.2782 $\pm$ 0.0017 & 0.3045 $\pm$ 0.0166 & 0.0308 $\pm$ 0.0004 & 0.0412 $\pm$ 0.0008 \\
 & \textbf{Ours} & \textbf{0.0589 $\pm$ 0.0027} & \textbf{0.0500 $\pm$ 0.0013} & \textbf{0.0284 $\pm$ 0.0008} & 0.0403 $\pm$ 0.0007 \\
\midrule
ASN & DiGress & 0.1496 & 0.3258 & 0.1506 & 0.4420 \\
 & LGGM-X & \textbf{0.0281} & 0.2440 & \textbf{0.0830} & \textbf{0.0618} \\
 & Ours-Mono & 0.1350 $\pm$ 0.0070 & 0.1375 $\pm$ 0.0053 & 0.1484 $\pm$ 0.0056 & 0.5197 $\pm$ 0.0411 \\
 & \textbf{Ours} & 0.0780 $\pm$ 0.0081 & 0.1281 $\pm$ 0.0096 & 0.1003 $\pm$ 0.0049 & 0.3012 $\pm$ 0.0254 \\
\midrule
EMAIL & DiGress & 0.2192 & 0.6012 & 0.0702 & 0.3416 \\
 & LGGM-X & 0.0751 & 0.2364 & 0.0768 & 0.3089 \\
 & Ours-Mono & 0.1396 $\pm$ 0.0017 & 0.0921 $\pm$ 0.0041 & 0.0600 $\pm$ 0.0011 & 0.0765 $\pm$ 0.0023 \\
 & \textbf{Ours} & \textbf{0.0631 $\pm$ 0.0037} & \textbf{0.0529 $\pm$ 0.0005} & \textbf{0.0563 $\pm$ 0.0021} & \textbf{0.0725 $\pm$ 0.0002} \\
\midrule
WEB & DiGress & 0.2556 & 0.6186 & 0.1877 & 0.6045 \\
 & LGGM-X & \textbf{0.0648} & 0.3961 & \textbf{0.0549} & \textbf{0.1127} \\
 & Ours-Mono & 0.1553 $\pm$ 0.0033 & 0.1456 $\pm$ 0.0038 & 0.1749 $\pm$ 0.0009 & 0.3319 $\pm$ 0.0096 \\
 & \textbf{Ours} & 0.0692 $\pm$ 0.0024 & \textbf{0.0795 $\pm$ 0.0021} & 0.1145 $\pm$ 0.0023 & 0.2760 $\pm$ 0.0076 \\
\midrule
ROAD & DiGress & 0.3705 & 0.8226 & 0.2801 & 0.7198 \\
 & LGGM-X & \textbf{0.0713} & 0.2193 & \textbf{0.0987} & \textbf{0.2986} \\
 & Ours-Mono & 0.4300 $\pm$ 0.0025 & 0.6884 $\pm$ 0.0091 & 0.3142 $\pm$ 0.0016 & 0.7433 $\pm$ 0.0076 \\
 & \textbf{Ours} & 0.1107 $\pm$ 0.0039 & \textbf{0.1850 $\pm$ 0.0140} & 0.1446 $\pm$ 0.0047 & 0.6027 $\pm$ 0.0194 \\
\midrule
POWER & DiGress & 0.3726 & 0.4582 & 0.3270 & 1.4732 \\
 & LGGM-X & \textbf{0.0119} & \textbf{0.1293} & \textbf{0.0373} & \textbf{0.0754} \\
 & Ours-Mono & 0.1963 $\pm$ 0.0037 & 0.3947 $\pm$ 0.0184 & 0.2809 $\pm$ 0.0036 & 0.9536 $\pm$ 0.0265 \\
 & \textbf{Ours} & 0.1660 $\pm$ 0.0063 & 0.2992 $\pm$ 0.0142 & 0.2265 $\pm$ 0.0066 & 0.6956 $\pm$ 0.0293 \\
\midrule
BIO & DiGress & 0.2712 & 0.5202 & 0.1127 & 0.3188 \\
 & LGGM-X & 0.1081 & 0.2696 & 0.0900 & 0.2053 \\
 & Ours-Mono & 0.0879 $\pm$ 0.0013 & 0.1467 $\pm$ 0.0078 & 0.0827 $\pm$ 0.0012 & 0.1584 $\pm$ 0.0077 \\
 & \textbf{Ours} & \textbf{0.0442 $\pm$ 0.0016} & \textbf{0.0605 $\pm$ 0.0033} & \textbf{0.0629 $\pm$ 0.0021} & \textbf{0.1090 $\pm$ 0.0053} \\
\midrule
ECON & DiGress & 0.2987 & 0.4841 & 0.2162 & 0.3834 \\
 & LGGM-X & 0.1213 & 0.0920 & 0.1120 & 0.1086 \\
 & Ours-Mono & 0.1011 $\pm$ 0.0011 & 0.1092 $\pm$ 0.0066 & 0.0707 $\pm$ 0.0010 & 0.0916 $\pm$ 0.0028 \\
 & \textbf{Ours} & \textbf{0.0780 $\pm$ 0.0018} & \textbf{0.0474 $\pm$ 0.0030} & \textbf{0.0199 $\pm$ 0.0007} & \textbf{0.0598 $\pm$ 0.0026} \\
\midrule
RT & DiGress & 0.4164 & \textbf{0.1327} & 0.4147 & 0.5957 \\
 & LGGM-X & \textbf{0.0525} & 0.1429 & \textbf{0.1330} & 0.2219 \\
 & Ours-Mono & 0.2911 $\pm$ 0.0040 & 0.4537 $\pm$ 0.0385 & 0.3705 $\pm$ 0.0039 & 0.3612 $\pm$ 0.0247 \\
 & \textbf{Ours} & 0.1999 $\pm$ 0.0081 & 0.4070 $\pm$ 0.0341 & 0.2345 $\pm$ 0.0120 & 0.2452 $\pm$ 0.0280 \\
\midrule
COL & DiGress & 0.2473 & 0.5826 & 0.2314 & 0.7679 \\
 & LGGM-X & \textbf{0.0736} & 0.5769 & \textbf{0.0895} & \textbf{0.0988} \\
 & Ours-Mono & 0.2143 $\pm$ 0.0035 & 0.3022 $\pm$ 0.0100 & 0.2329 $\pm$ 0.0020 & 0.7161 $\pm$ 0.0080 \\
 & \textbf{Ours} & 0.1020 $\pm$ 0.0021 & \textbf{0.1683 $\pm$ 0.0062} & 0.1863 $\pm$ 0.0034 & 0.6032 $\pm$ 0.0156 \\
\midrule
ECO & DiGress & 0.5431 & 0.7915 & 0.2338 & 0.6045 \\
 & LGGM-X & 0.4753 & 0.3904 & 0.3194 & 0.3934 \\
 & Ours-Mono & 0.1744 $\pm$ 0.0010 & 0.3025 $\pm$ 0.0038 & 0.0436 $\pm$ 0.0005 & 0.2044 $\pm$ 0.0001 \\
 & \textbf{Ours} & \textbf{0.0632 $\pm$ 0.0026} & \textbf{0.2298 $\pm$ 0.0003} & \textbf{0.0326 $\pm$ 0.0004} & 0.2040 $\pm$ 0.0000 \\
\midrule
CITATION & DiGress & 0.2527 & 0.7790 & 0.1315 & 0.4966 \\
 & LGGM-X & 0.1348 & 0.7257 & 0.1160 & 0.4981 \\
 & Ours-Mono & 0.1364 $\pm$ 0.0013 & 0.2654 $\pm$ 0.0093 & 0.0917 $\pm$ 0.0007 & 0.1601 $\pm$ 0.0052 \\
 & \textbf{Ours} & \textbf{0.0407 $\pm$ 0.0013} & \textbf{0.1150 $\pm$ 0.0019} & \textbf{0.0621 $\pm$ 0.0017} & \textbf{0.1010 $\pm$ 0.0026} \\
\midrule
ALL & DiGress & 0.3112 & 0.5622 & 0.2030 & 0.5923 \\
 & LGGM-X & 0.1408 & 0.3422 & 0.1253 & \textbf{0.2616} \\
 & Ours-Mono & 0.1950 $\pm$ 0.0027 & 0.2785 $\pm$ 0.0111 & 0.1584 $\pm$ 0.0019 & 0.3632 $\pm$ 0.0114 \\
 & \textbf{Ours} & \textbf{0.0895 $\pm$ 0.0037} & \textbf{0.1519 $\pm$ 0.0075} & \textbf{0.1057 $\pm$ 0.0035} & 0.2759 $\pm$ 0.0114 \\
\bottomrule
\end{tabular}
\par\smallskip
\begin{minipage}{\linewidth}\footnotesize\raggedright
DEG, CC, Spec, Orb: MMD of degree, clustering coefficient, Laplacian eigenvalues and orbit counts. DiGress and LGGM-X are published values. Ours and Ours-Mono are mean $\pm$ SD over 10 seeds; the SD of the ALL row is the mean of the twelve per-domain SDs. Bold marks the best method where the mean $\pm$ SD intervals do not overlap. On ECO the orbit-count kernel saturates, so the orbit MMD is a constant set by the reference graphs for any distinct, dense generated graphs; it is not used to mark the best method.
\end{minipage}
\end{table}

\section{Single-instance reconstruction: protocol}
\label{app:single_instance}

This appendix states the protocol behind Table~\ref{tab:wmgm_compare}. Its metrics are defined in Appendix~\ref{app:metrics_definitions}. Each dataset is a single observed network. The comparison asks the reconstruction fidelity of networks generated from the zero-shot estimate $\hat\theta^{(1)}$, with a per-instance expectation--maximisation fit (WMGM) as a reference for what direct optimisation on the same network achieves.

\paragraph{Data.} We use six networks, each observed once. Four are directed (C.~elegans, Drosophila medulla, Allen mouse and P2P-Gnutella). Two are undirected (Power Grid and Minnesota Roads). Only C.~elegans and Allen mouse carry edge weights. The C.~elegans file is a $383 \times 575$ connection table whose rows are the 383 neurons. We keep only the neuron-to-neuron connections, the leading $383 \times 383$ block. The Allen mouse, Drosophila and C.~elegans networks contain self-loops (163, 105 and 8), which are kept. They enter the edge density, the node strengths used by the multifractal MSE and the degrees, where each self-loop counts once. Before use, we divide every weight by the smallest non-zero weight in that network, so the weakest edge has weight 1. We then order the nodes as described in Appendix~\ref{app:ordering}. This full-size network is the reference that every generated network is compared against. Both methods receive this network after log-binning (Eq.~\eqref{eq:log_binning}). WMGM is fitted on the full network. Our operator receives the crop to the largest dyadic size for compatibility with multiwavelet transform, as described in Appendix~\ref{app:ordering}.

\paragraph{Recursion depth.} As in Appendix~\ref{app:lggm_protocol}, $K^\star$ is set by the closed-form density relation from $\hat\theta^{(1)}$ and the observed edge density, and is fixed before generation. For WMGM, which is fitted separately at each depth $K\in\{2,\dots,6\}$, the reported depth is the one at which the fitted parameters are self-consistent with the same relation, i.e.\ the $K$ minimising $|K^\star(\hat\theta^{(1)}_K)-K|$. No evaluation metric enters the choice for either method.

\paragraph{Generation and evaluation.} For each of ten seeds (42--51), one weighted network of the original size is sampled from each method's parameters at its depth. Samples are directed and may contain self-loops. The degree and spectral distances are computed on the binarized sample with self-loops kept. The sample is symmetrised by union for the spectral distance and, for the undirected references (Power Grid and Minnesota Roads), also for the degree distance. The multifractal MSE uses the out-strengths of the directed weighted sample. Each sample is compared with the reference network using multifractal MSE, normalised degree Wasserstein distance, and spectral Wasserstein distance of the normalised Laplacian. We report mean and standard deviation over seeds.

\paragraph{WMGM fitting.} WMGM is run on the same canonicalised, log-binned matrix at full size with $M=2$, up to 1000 EM iterations and an Adam step size of $10^{-3}$, once per depth and seed. Fitting stops early when the change in the evidence lower bound falls below $10^{-4}$, and the step size is halved when the objective plateaus (patience 50, cooldown 20, minimum $10^{-6}$). The parameters are initialized as follows: $l^{(1)}$ as uniform, $p^{(1)}$ from $\mathrm{Uniform}(0,1)$, and the community responsibilities from $\mathrm{Dirichlet}(1)$. $p^{(1)}$ is clamped to $[10^{-6}, 1-10^{-6}]$. Ties in the depth rule are resolved to the smallest $K$. Its reported time is the mean wall-clock time per fit over the five depths and ten seeds, measured on the same GPU as our inference.

\begin{table}[t]
\centering
\caption{Full version of Table~\ref{tab:wmgm_compare} with seed-to-seed standard deviations.}
\label{tab:wmgm_full}
\small
\setlength{\tabcolsep}{4pt}
\begin{tabular}{llcccc}
\toprule
Domain & Method & Time (ms) & MFrac & DegW & SpecW \\
\midrule
C. elegans & WMGM & 2060 $\pm$ 436 & 2.68e-2 $\pm$ 2.17e-2 & \textbf{5.67e-4 $\pm$ 2.55e-4} & \textbf{6.97e-2 $\pm$ 3.47e-3} \\
 & Ours & \textbf{3.76 $\pm$ 0.97} & \textbf{2.19e-4 $\pm$ 4.02e-5} & 1.20e-3 $\pm$ 2.08e-5 & 1.27e-1 $\pm$ 6.92e-4 \\
\midrule
Drosophila & WMGM & 1930 $\pm$ 77 & 7.35e-3 $\pm$ 8.56e-3 & 2.29e-4 $\pm$ 1.18e-4 & \textbf{6.41e-2 $\pm$ 1.62e-2} \\
 & Ours & \textbf{14.1 $\pm$ 0.03} & 3.84e-3 $\pm$ 3.68e-4 & 2.69e-4 $\pm$ 6.71e-6 & 1.18e-1 $\pm$ 1.42e-3 \\
\midrule
Allen Mouse & WMGM & 1050 $\pm$ 243 & \textbf{1.52e-4 $\pm$ 1.51e-5} & 1.38e-4 $\pm$ 3.76e-5 & 3.99e-3 $\pm$ 9.20e-3 \\
 & Ours & \textbf{2.95 $\pm$ 0.04} & 7.86e-3 $\pm$ 2.00e-3 & 9.86e-5 $\pm$ 8.96e-6 & 6.82e-3 $\pm$ 8.37e-4 \\
\midrule
P2P Gnutella & WMGM & 2540 $\pm$ 131 & 5.19e-3 $\pm$ 3.07e-3 & 1.41e-4 $\pm$ 2.78e-5 & \textbf{8.09e-2 $\pm$ 2.23e-2} \\
 & Ours & \textbf{198 $\pm$ 0.22} & 3.56e-3 $\pm$ 3.43e-4 & 1.23e-4 $\pm$ 8.12e-7 & 1.24e-1 $\pm$ 9.60e-4 \\
\midrule
Power Grid & WMGM & 1810 $\pm$ 201 & 5.88e-4 $\pm$ 9.86e-4 & 4.23e-5 $\pm$ 1.45e-5 & \textbf{1.81e-1 $\pm$ 6.52e-2} \\
 & Ours & \textbf{198 $\pm$ 0.05} & 1.24e-3 $\pm$ 8.78e-5 & \textbf{2.54e-5 $\pm$ 4.88e-7} & 2.58e-1 $\pm$ 2.25e-3 \\
\midrule
MN Roads & WMGM & 1490 $\pm$ 280 & \textbf{1.73e-3 $\pm$ 2.64e-4} & \textbf{5.92e-5 $\pm$ 6.16e-6} & 1.86e-1 $\pm$ 4.97e-2 \\
 & Ours & \textbf{50.9 $\pm$ 1.59} & 2.42e-3 $\pm$ 2.00e-5 & 1.24e-4 $\pm$ 1.91e-6 & 1.84e-1 $\pm$ 5.11e-3 \\
\bottomrule
\end{tabular}
\par\smallskip
\begin{minipage}{\linewidth}\footnotesize\raggedright
MFrac: multifractal MSE; DegW: degree Wasserstein; SpecW: spectral Wasserstein. Values are mean $\pm$ SD over 10 seeds. Bold marks the lower value where the two methods' mean $\pm$ SD intervals do not overlap.
\end{minipage}
\end{table}

\section{Definitions of Evaluation Metrics for Single Observation Datasets}
\label{app:metrics_definitions}

\subsection{Multifractal Spectrum ($f(\alpha)$) Estimation}
To estimate the continuous multifractal spectrum of the networks, we apply the Chhabra-Jensen method. Let $A$ be the adjacency matrix of a network with $N$ active nodes, and let $d_i = \sum_j A_{ij}$ be the degree of node $i$. We first define the degree probability distribution $P_i = d_i / \sum d_i$.

To isolate different scaling regimes, we define a one-parameter family of normalized measures, $\mu_i(q)$, parameterized by the moment $q \in \mathbb{R}$:
\begin{equation} \label{eq:mu_q}
    \mu_i(q) = \frac{P_i^q}{\sum_{j=1}^N P_j^q}
\end{equation}
By varying $q$, we heavily weight the massive hubs ($q > 0$) or the low-degree periphery ($q < 0$). The Lipschitz-H\"older exponent $\alpha(q)$ and its corresponding fractal dimension $f(\alpha(q))$ are then computed parametrically as:
\begin{equation} \label{eq:alpha_q}
    \alpha(q) = -\frac{1}{\log N} \sum_{i=1}^N \mu_i(q) \log P_i
\end{equation}
\begin{equation} \label{eq:f_alpha}
    f(\alpha(q)) = -\frac{1}{\log N} \sum_{i=1}^N \mu_i(q) \log \mu_i(q)
\end{equation}
The \textbf{Multifractal MSE} is calculated as the mean squared error between the empirical and synthetic continuous curves, $f_{emp}(\alpha)$ and $f_{syn}(\alpha)$, interpolated over a shared $\alpha$ support space. We use 30 values of $q$ evenly spaced in $[-4, 4]$ and interpolate both curves linearly on 50 points spanning the overlap of their $\alpha$ ranges.

\subsection{Degree Wasserstein Distance}
To evaluate the absolute local hub structure, we utilize the 1-Wasserstein distance (Earth Mover's Distance) between the sets of empirical and synthetic degrees. Let $U(x)$ and $V(x)$ be the cumulative distribution functions of the empirical degrees $d^{emp}$ and synthetic degrees $d^{syn}$, respectively. The unscaled degree distance is defined as:
\begin{equation} \label{eq:deg_wasserstein}
    W_1(d^{emp}, d^{syn}) = \int_{-\infty}^{\infty} |U(x) - V(x)| \, dx
\end{equation}
For cross-dataset comparisons independent of absolute scale (as presented in the main text), we compute the Normalized Degree Wasserstein by transforming the degrees into probability mass functions: $\tilde{d}_i = d_i / \sum d_j$ prior to calculating the Wasserstein integral.

\subsection{Spectral Wasserstein Distance}
To evaluate the global community structure, bottlenecking, and information flow of the networks, we analyze the eigenspectrum of the Normalized Laplacian matrix. The Normalized Laplacian is defined as:
\begin{equation} \label{eq:norm_laplacian}
    \mathcal{L} = I - D^{-1/2} A D^{-1/2}
\end{equation}
where $I$ is the identity matrix and $D$ is the diagonal degree matrix. Because the eigenvalues of $\mathcal{L}$ are strictly bounded within the interval $[0, 2]$ regardless of network size or absolute density, they provide a perfectly scaled invariant measure of global topology.

Let $\lambda^{emp}$ and $\lambda^{syn}$ denote the sorted sets of eigenvalues derived from the empirical and synthetic Normalized Laplacians. The Spectral Wasserstein distance is calculated as the 1-Wasserstein distance between these two discrete distributions:
\begin{equation} \label{eq:spec_wasserstein}
    W_1(\lambda^{emp}, \lambda^{syn}) = \inf_{\gamma \in \Gamma(\lambda^{emp}, \lambda^{syn})} \int \int |x - y| \, d\gamma(x, y)
\end{equation}
where $\Gamma$ denotes the collection of all joint probability measures with marginals equivalent to the empirical and synthetic eigenvalue distributions.

The degree and spectral distances are computed on binarized graphs. For directed networks, the degree distance uses out-degrees and the spectral distance uses the union-symmetrised graph. Standard deviations over seeds (Tables~\ref{tab:lggm_full} and~\ref{tab:wmgm_full}) are population SDs; Table~\ref{tab:eeg_alpha_groups} reports sample SDs across subjects.

\section{EEG Analysis}
\label{app:eeg}

\paragraph{Data.} We use the propofol sedation dataset of \citet{chennu2016brain}, which includes 20 healthy volunteers who were recorded at rest with 91-channel scalp EEG (250\,Hz) across four sessions: before propofol administration (baseline), under mild sedation, under moderate sedation, and after recovery.

\paragraph{Network construction.} Each epoch is band-pass filtered to the alpha band (8--15\,Hz, fourth-order Butterworth, zero-phase) and Hilbert-transformed. Functional connectivity between every pair of channels is the weighted phase-lag index (wPLI) computed per epoch and averaged over the epochs of the session \citep{vinck2011improved}. Edges are retained by a surrogate test. For each epoch we generate 1000 phase-randomised surrogates of every channel (Fourier amplitudes preserved, phases drawn uniformly), compute their wPLI matrices, and average them over epochs. An edge is kept, with its wPLI value as weight, if the observed session-mean wPLI exceeds the 95th percentile of the surrogate distribution for that channel pair.

\begin{figure}[t]
    \centering
    \begin{minipage}[b]{0.40\linewidth}
        \centering
        \includegraphics[width=\linewidth]{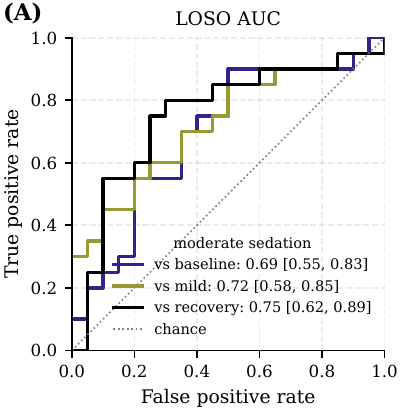}
    \end{minipage}
    \hfill
    \begin{minipage}[b]{0.52\linewidth}
        \centering
        \includegraphics[width=\linewidth]{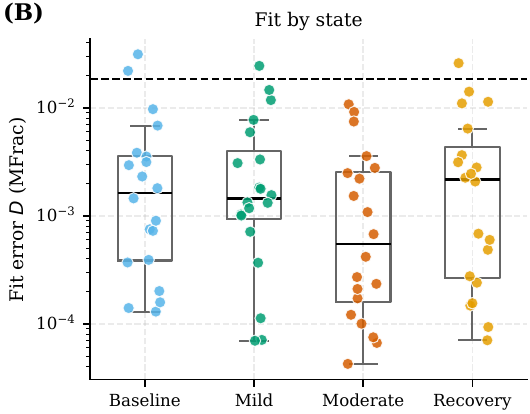}
    \end{minipage}
    \\[2mm]
    \includegraphics[width=\linewidth]{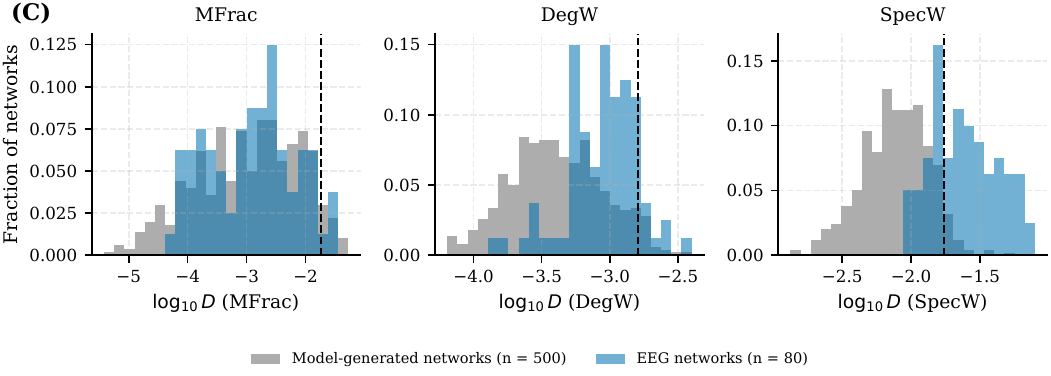}
    \caption{\textbf{Supplementary EEG analyses.}
    \textbf{(A)} Predictive validity. ROC curves of the leave-one-subject-out (LOSO) logistic regression. The legend gives the AUC with its 95\% bootstrap
    interval over subjects.  Dotted line represents chance.
    \textbf{(B)} Fit by state. Reconstruction error $D$ (multifractal MSE, mean over ten generated networks) of every
    EEG network, grouped by sedation state ($n=20$ per state). Boxes represent median and interquartile range. Whiskers represent $1.5\times$IQR. Dashed line represents the pass threshold of (C) $D = 0.0185$.
    (Kruskal--Wallis $p = 0.40$; Friedman over subjects $p = 0.42$; two-sided Mann--Whitney tests of moderate sedation
    against each other state, Holm-corrected, $p = 0.65$ for all three; baseline vs recovery $p = 0.86$).
    \textbf{(C)} Synthetic reference range. Distribution of $\log_{10} D$ for 500 model-generated networks whose
    density lies in the EEG range (gray) and for the 80 EEG networks (blue), for the multifractal MSE (MFrac; the
    decision metric), the normalised degree Wasserstein distance (DegW) and the spectral Wasserstein distance (SpecW).
    Dashed lines represent the 95th percentile of the model-generated networks ($0.0185$, $0.0016$ and $0.017$). The number of EEG networks
    below the line pass are as follows: 76 of 80 on MFrac, 71 of 80 on DegW, and 27 of 80 on SpecW.}
    \label{fig:eeg_supp}
\end{figure}

\paragraph{Inference.} Each network is processed as in Appendix~\ref{app:single_instance}. Weights are normalised so that the weakest edge has weight 1, nodes are canonicalised (Appendix~\ref{app:ordering}), the matrix is cropped to 64 nodes, which removes the 27 lowest-strength channels, and log-binned. One forward pass of the operator returns $\hat\theta^{(1)}$. From $\hat\theta^{(1)}=(\hat l^{(1)},\hat p^{(1)})$ we report two scalar summaries. One is the assortativity index $\hat p_{00}+\hat p_{11}-\hat p_{01}-\hat p_{10}$, which is positive when within-class connections dominate. The other is multifractal width $\Delta\alpha=\log_2(\hat p_{\max}/\hat p_{\min})$, which is the width of the spectrum of Appendix~\ref{app:closed-form} for $M=2$. Classical metrics (mean clustering, mean shortest path, greedy modularity, spectral gap of the normalised Laplacian) are computed on the same networks before cropping and after binarisation ($A_{ij}>0$). The shortest path and spectral gap use the largest connected component.

\paragraph{Group comparison.} The four sessions are compared as groups of 20 networks. For each measure we report the following: 1) mean and SD per state 2) Cohen's $d$ with pooled SD for each state minus moderate sedation and for baseline minus recovery 3) a two-sided Mann--Whitney $U$ test per contrast with Holm correction over the four contrasts.

\paragraph{Predictive validity.} To test whether the inferred rule summaries generalize to new individuals, we train a logistic regression on the two $\hat\theta^{(1)}$-derived summaries to distinguish moderate sedation from each other state, holding out one subject at a time. Features are standardised within each training fold, and the model uses the default L2 penalty ($C=1$). Predictions are pooled across held-out subjects and scored by AUC with 95\% bootstrap intervals over the 20 subjects (1000 resamples). The AUC is $0.69$ [$0.55$, $0.83$] for moderate sedation versus baseline, $0.72$ [$0.58$, $0.85$] versus mild sedation and $0.75$ [$0.62$, $0.89$] versus recovery. All three intervals exclude chance ($0.5$).

\paragraph{Synthetic reference range.} We ask whether the reconstruction error of each EEG network is within the range seen when the model is true. We draw parameter sets from the training prior and a depth $K$ uniformly from $\{2,\dots,6\}$, and generate one 91-node network from each. Unlike the EEG networks, which are symmetric, loop-free and wPLI-weighted, these networks are directed, may contain self-loops and have integer weights. As in Table~\ref{tab:wmgm_compare}, the degree and spectral distances are computed on binarized, union-symmetrised graphs with self-loops kept, and the multifractal MSE on the directed weighted samples. We keep a network only if its edge density lies in the observed EEG range, until 500 networks are kept. Each network, synthetic or EEG, is then processed as described above, given to MWNOT as inputs, and ten networks are generated from each inferred rule (seeds 42 to 51). The generated graphs are each compared with the input network using the multifractal MSE of Appendix~\ref{app:metrics_definitions}. The mean over the ten comparisons is the error $D$. The pass threshold is the 95th percentile of $D$ over the 500 synthetic networks, $0.0185$. The model's fit on a given EEG is considered acceptable if its $D$ is below the threshold.

Of the 80 EEG networks, 76 pass (18 baseline, 19 mild, 20 moderate and 19 recovery). The median $D$ of the EEG networks, $0.00139$, is close to the synthetic median, $0.00130$. The degree Wasserstein and spectral Wasserstein distances are reported for completeness but do not enter the decision. Under the same rule, 71 networks pass on the degree distance and 27 on the spectral distance. The spectral gap between real and synthetic networks is consistent with Table~\ref{tab:wmgm_compare}, where spectral error is also the weakest term.

\paragraph{Fit and sedation state.} State differences in the inferred rule could reflect worse fit under moderate sedation rather than a real change. To rule this out, we compared fit across states. We use $D$ as the fit score. Its median is $0.0016$ at baseline, $0.0014$ under mild sedation, $0.0005$ under moderate sedation and $0.0022$ at recovery. A Kruskal--Wallis test across the four states gives $p = 0.40$. A Friedman test over the 20 subjects gives $p = 0.42$. Two-sided Mann--Whitney tests of moderate sedation against each other state, Holm-corrected over the four contrasts of Table~\ref{tab:eeg_alpha_groups}, give $p = 0.65$ for all three. Baseline versus recovery gives $p = 0.86$. We found no evidence that fit differs across states.

\section{Compression of a periodic multirule generator}
\label{app:compression}

\begin{proposition}[Compression of a periodic multirule generator]
\label{prop:cycle_compression}
For each scale \(k\), let
\[
l^{[k]} \in \mathrm{int}(\Delta^{M_k-1}),
\qquad
p^{[k]} \in (0,1)^{M_k\times M_k},
\]
and define the multirule recursive expansion
\[
l^{(K)} = \bigotimes_{k=1}^K l^{[k]},
\qquad
p^{(K)} = \bigotimes_{k=1}^K p^{[k]}.
\]
Assume that the rule sequence is periodic with period \(R\), so that
\[
(l^{[k+R]},p^{[k+R]})=(l^{[k]},p^{[k]})
\qquad\text{for all }k,
\]
and let \(K=mR\). Define
\[
\tilde l^{(1)} := \bigotimes_{r=1}^R l^{[r]},
\qquad
\tilde p^{(1)} := \bigotimes_{r=1}^R p^{[r]},
\qquad
\tilde M := \prod_{r=1}^R M_r.
\]
Then
\[
\tilde l^{(1)} \in \mathrm{int}(\Delta^{\tilde M-1}),
\qquad
\tilde p^{(1)} \in (0,1)^{\tilde M\times \tilde M},
\]
and the original multirule generator is exactly representable as
\[
l^{(K)} = \underbrace{\tilde l^{(1)}\otimes \cdots \otimes \tilde l^{(1)}}_{m\text{ times}},
\qquad
p^{(K)} = \underbrace{\tilde p^{(1)}\otimes \cdots \otimes \tilde p^{(1)}}_{m\text{ times}}.
\]
Hence a periodic multirule recursive generator reduces exactly to a
single-basis generator with effective branching number
\[
\tilde M = \prod_{r=1}^R M_r.
\]
\end{proposition}

\begin{proof}
Since each \(l^{[r]}\) is a strictly positive probability vector,
their Kronecker product \(\tilde l^{(1)}\) is again a strictly positive
probability vector, now of dimension \(\prod_{r=1}^R M_r\). Likewise,
\(\tilde p^{(1)}\) has dimension \(\tilde M\times \tilde M\) and entries in
\((0,1)\).

Because the rule sequence is periodic and \(K=mR\), the full expansion
can be grouped into \(m\) identical blocks:
\[
l^{(K)}
=
\Bigl(\bigotimes_{r=1}^R l^{[r]}\Bigr)
\otimes
\cdots
\otimes
\Bigl(\bigotimes_{r=1}^R l^{[r]}\Bigr)
=
(\tilde l^{(1)})^{\otimes m}.
\]
The same argument gives
\[
p^{(K)} = (\tilde p^{(1)})^{\otimes m}.
\]
Therefore the multirule generator is exactly representable as an
\(m\)-level recursion of the single effective basis
\((\tilde l^{(1)},\tilde p^{(1)})\).
\end{proof}

\textbf{Obstruction to finite single-basis reduction:}
\label{cor:nonperiodic_obstruction}
If a multirule recursive generator does not admit a decomposition into
repeated identical finite blocks, then it cannot be represented exactly
by a finite-dimensional single basis repeated across scales.

\begin{proof}
Any finite-dimensional single basis repeated across scales induces a rule
sequence obtained by repetition of one finite block. The claim follows by
contrapositive.
\end{proof}

\begin{theorem}[Minimal effective branching number]
\label{thm:minimal_effective_branching}
Consider a multirule recursive generator of depth \(K\),
\[
l^{(K)}=\bigotimes_{k=1}^K l^{[k]},
\qquad
p^{(K)}=\bigotimes_{k=1}^K p^{[k]},
\]
where
\[
l^{[k]} \in \mathrm{int}(\Delta^{M_k-1}),
\qquad
p^{[k]} \in (0,1)^{M_k\times M_k}.
\]
Assume that the rule sequence admits an exact single-basis reduction.
Let \(q_\ast\) be the smallest positive divisor of \(K\) such that the sequence
\[
\bigl((l^{[1]},p^{[1]}),\dots,(l^{[K]},p^{[K]})\bigr)
\]
can be partitioned into \(K/q_\ast\) consecutive identical blocks of length \(q_\ast\),
up to one consistent permutation of community labels.

Then the smallest effective branching number among all exact single-basis reductions is
\[
\tilde M_{\min}=\prod_{k=1}^{q_\ast} M_k.
\]
Moreover, the corresponding minimal effective basis is
\[
\tilde l_{\min}^{(1)}=\bigotimes_{k=1}^{q_\ast} l^{[k]},
\qquad
\tilde p_{\min}^{(1)}=\bigotimes_{k=1}^{q_\ast} p^{[k]},
\]
up to permutation of the \(\tilde M_{\min}\) effective labels, and satisfies
\[
l^{(K)} = \bigl(\tilde l_{\min}^{(1)}\bigr)^{\otimes K/q_\ast},
\qquad
p^{(K)} = \bigl(\tilde p_{\min}^{(1)}\bigr)^{\otimes K/q_\ast}.
\]
\end{theorem}

\begin{proof}
Let \(q_\ast\) be the shortest repeating block length. Grouping each block into a
single effective refinement step gives
\[
\tilde l_{\min}^{(1)}=\bigotimes_{k=1}^{q_\ast} l^{[k]},
\qquad
\tilde p_{\min}^{(1)}=\bigotimes_{k=1}^{q_\ast} p^{[k]},
\]
which have effective dimension
\[
\tilde M_{\min}=\prod_{k=1}^{q_\ast} M_k.
\]
Since the sequence consists of \(K/q_\ast\) identical consecutive blocks,
\[
l^{(K)} = \bigl(\tilde l_{\min}^{(1)}\bigr)^{\otimes K/q_\ast},
\qquad
p^{(K)} = \bigl(\tilde p_{\min}^{(1)}\bigr)^{\otimes K/q_\ast}.
\]
Thus an exact single-basis reduction with branching number \(\tilde M_{\min}\) exists.

Now let
\[
l^{(K)} = (\hat l^{(1)})^{\otimes J},
\qquad
p^{(K)} = (\hat p^{(1)})^{\otimes J}
\]
be any other exact single-basis reduction. Then the original rule sequence must
decompose into \(J\) identical consecutive blocks, each of length \(q=K/J\). By
minimality of \(q_\ast\), one has \(q\ge q_\ast\). Therefore no exact reduction
can use a shorter repeating block than \(q_\ast\), and the construction above
attains the smallest possible effective branching number. Hence
\[
\tilde M_{\min}=\prod_{k=1}^{q_\ast} M_k.
\]
\end{proof}

\newpage

\end{document}